%% file: main-info-reuse.tex
\newif\iffull
\fulltrue

\documentclass{article}
\iffull
\usepackage{fullpage}
\else
\usepackage{nips15submit_e,times}
\fi
\usepackage{amsmath,amsthm,amssymb,xspace}
\usepackage{graphicx}
\usepackage{color}
\usepackage{algorithmic}
\usepackage{algorithm}
\usepackage{microtype}
\usepackage[small]{caption}
\def\draft{1}

\ifnum\draft=1 
    \def\ShowAuthNotes{1}
\else
    \def\ShowAuthNotes{0}
\fi

\usepackage{hyperref}

\input{acm-macs}

\usepackage{vitalymacros}

\usepackage{boxedminipage}
\usepackage{bm}
\newcommand{\Esymb}{\mathbb{E}}
\newcommand{\Psymb}{\mathbb{P}}
\DeclareMathOperator*{\Expect}{\Esymb}
\DeclareMathOperator*{\Prob}{\Psymb}
\renewcommand{\E}{\Expect}
\renewcommand{\pr}{\Prob}
\renewcommand{\Pr}{\Prob}
\newcommand{\cE}{{\cal E}}
\usepackage{listings}

\newcommand{\bX}{{\bm X}}
\providecommand{\X}{{\cal X}}
\newcommand{\bY}{{\bm Y}}
\newcommand{\bZ}{{\bm Z}}

\newcommand{\Y}{{\cal Y}}

\newcommand{\bS}{{\bm S}}

\newcommand{\bphi}{{\bm \phi}}
\newcommand{\ba}{{\bm a}}

\providecommand{\cO}{{\mathcal O}}
\providecommand{\cQ}{{\mathcal Q}}
\providecommand{\cS}{{\mathcal S}}
\iffull
\newcommand{\Tho}{{\sf Thresholdout}\xspace}
\newcommand{\Sparse}{{\sf SparseValidate}\xspace}
\else
\newcommand{\Tho}{{\footnotesize{\sf Thresholdout}}\xspace}
\newcommand{\Sparse}{{\footnotesize{\sf SparseValidate}}\xspace}
\fi

\iffull
\newcommand{\para}[1]{\paragraph{#1}}
\else
\newcommand{\para}[1]{\noindent{\bf #1}}
\fi

\begin{document}
\title{Generalization in Adaptive Data Analysis and Holdout Reuse}
\author{Cynthia Dwork\thanks{Microsoft Research} \and Vitaly Feldman\thanks{IBM Almaden Research Center. Part of this work done while visiting the Simons Institute, UC Berkeley}  \and Moritz Hardt\thanks{Google Research} \and Toniann Pitassi\thanks{University of Toronto} \and Omer Reingold\thanks{Samsung Research America} \and Aaron Roth\thanks{Department of Computer and Information Science, University of Pennsylvania}}

\maketitle

\begin{abstract}
Overfitting is the bane of data analysts, even when data are plentiful. Formal approaches to understanding this problem focus on statistical inference and generalization of individual analysis procedures. Yet the practice of data analysis is an inherently interactive and adaptive process: new analyses and hypotheses are proposed after seeing the results of previous ones, parameters are tuned on the basis of obtained results, and datasets are shared and reused.  An investigation of this gap has recently been initiated by the authors in \cite{DworkFHPRR14:arxiv}, where we focused on the problem of estimating expectations of adaptively chosen functions.

In this paper, we give a simple and practical method for reusing a holdout (or testing) set to validate the accuracy of hypotheses produced by a learning algorithm operating on a training set. Reusing a  holdout set adaptively multiple times can easily lead to overfitting to the holdout set itself. We give an algorithm that enables the validation of a large number of adaptively chosen hypotheses, while provably avoiding overfitting. We illustrate the advantages of our algorithm over the standard use of the holdout set via a simple synthetic experiment.

We also formalize and address the general problem of data reuse in adaptive data analysis. We show how the differential-privacy based approach given in \cite{DworkFHPRR14:arxiv} is applicable much more broadly to adaptive data analysis. We then show that a simple approach based on description length can also be used to give guarantees of statistical validity in adaptive settings. Finally, we demonstrate that these incomparable approaches can be unified via the notion of approximate max-information that we introduce. This, in particular, allows the preservation of statistical validity guarantees even when an analyst adaptively composes algorithms which have guarantees based on either of the two approaches.

\end{abstract}

\input{intro-info-reuse.tex}

\input{prelims-info-reuse.tex}
\iffull
\input{dp-info-reuse.tex}
\input{descr-info-reuse.tex}
\fi
\input{max-info-reuse.tex}

\input{holdout-info-reuse.tex}

\iffull
\input{experiment-info-reuse.tex}

\fi

\iffull
\section{Conclusions}
In this work, we give a unifying view of two techniques (differential privacy and description length bounds) which preserve the generalization guarantees of subsequent algorithms in adaptively chosen sequences of data analyses. Although these two techniques both imply low max-information -- and hence can be composed together while preserving their guarantees -- the kinds of guarantees that can be achieved by either alone are incomparable. This suggests that the problem of generalization guarantees under adaptivity is ripe for future study on two fronts. First, the existing theory is likely already strong enough to develop practical algorithms with rigorous generalization guarantees, of which \Tho is an example. However additional empirical work is needed to better understand when and how the theory should be applied in specific application scenarios. At the same time, new theory is also needed. As an example of a basic question we still do not know the answer to: even in the simple setting of adaptively reusing a holdout set for computing the expectations of boolean-valued predicates, is it possible to obtain stronger generalization guarantees (via any means) than those that are known to be achievable via differential privacy?
\fi
\iffull
\bibliographystyle{alpha}
\else
\bibliographystyle{plain}
\fi
\bibliography{../refs}
\appendix
\iffull
\input{maxinto2dl-info-reuse.tex}
\input{desclengthquerybounds.tex}
\fi
\end{document}

%% file: acm-macs.tex
\newcommand{\cS}{\mathcal{S}}

\newcommand{\eps}{\varepsilon}

\newcommand{\R}{\mathbb{R}}

\newcommand{\A}{{\cal A}}

\newcommand{\cC}{{\cal C}}
\newcommand{\cX}{{\cal X}}
\newcommand{\cY}{{\cal Y}}
\newcommand{\cZ}{{\cal Z}}

\newcommand{\eg}{{\it e.g.}}

\newcommand{\calP}{\mathcal {P}}

\newcommand{\D}{{\cal D}}

\newcommand{\X}{\mathcal{X}}

\newcommand{\zo}{\{0,1\}}

\newcommand{\pr}{\Pr} 

\newcommand{\sgn}{\textrm{sgn}}

\newcommand{\remove}[1]{}
\newcommand{\op}[1]{}

\newtheorem{theorem}{Theorem}

\newtheorem{corollary}[theorem]{Corollary}

\newtheorem{lemma}[theorem]{Lemma}

\newtheorem{definition}[theorem]{Definition}

\newtheorem{remark}[theorem]{Remark}

\definecolor{darkred}{rgb}{0.5,0,0}
\definecolor{brown}{rgb}{0.5,0.25,0}
\definecolor{darkgreen}{rgb}{0,.5,0}
\definecolor{lightgray}{gray}{.8}
\definecolor{lightrose}{rgb}{1, .7, .7}
\definecolor{lightskyblue}{rgb}{.9,.9,1}
\definecolor{aquamarine}{rgb}{.442,1,.812}
\definecolor{lightgreen}{rgb}{.7,1,.7}
\definecolor{lightpink}{rgb}{1,.8,0.8}
\definecolor{thistle}{rgb}{0.9,0.749,0.9}
\definecolor{lightblue}{rgb}{0.67843,0.847,0.9}
\ifnum\ShowAuthNotes=1
\newcommand{\authnote}[2]{{ \footnotesize \bf{\color{darkred}[#1's Note:
{\color{darkgreen}#2}]}}}
\else
\newcommand{\authnote}[2]{}
\fi

%% file: intro-info-reuse.tex
\section{Introduction}
The goal of machine learning is to produce hypotheses or models that generalize well to the unseen instances of the problem. More generally, statistical data analysis is concerned with estimating properties of the underlying data distribution, rather than properties that are specific to the finite data set at hand. Indeed, a large body of theoretical and empirical research was developed for ensuring generalization in a variety of settings. In this work, it is commonly assumed that each analysis procedure (such as a learning algorithm) operates on a freshly sampled dataset -- or if not, is validated on a freshly sampled \emph{holdout} (or testing) set.

Unfortunately, learning and inference can be more difficult in practice, where
data samples are often reused. For example, a common practice is to perform feature selection on a dataset, and then use the features for some supervised learning task. When these two steps are performed on the same dataset, it is no longer clear that the results obtained from the combined algorithm will generalize. Although not usually understood in these terms, ``Freedman's paradox" is an elegant demonstration
of the powerful (negative) effect of adaptive analysis on the same data \cite{Freedman83}. In Freedman's simulation, variables with significant $t$-statistic are selected and linear regression is performed on this adaptively chosen subset of variables, with famously misleading results: when the relationship between the dependent and explanatory variables is non-existent, the procedure overfits, erroneously declaring significant relationships.

Most of machine learning practice does not rely on formal guarantees of generalization for learning algorithms. Instead a dataset is split randomly into two (or sometimes more) parts: the training set and the testing, or  holdout, set. The training set is used for learning a predictor, and then the holdout set is used to estimate the accuracy of the predictor on the true distribution\footnote{Additional averaging over different partitions is used in cross-validation.}. Because the predictor is independent of the holdout dataset, such an estimate is a valid estimate of the true prediction accuracy (formally, this allows one to construct a confidence interval for the prediction accuracy on the data distribution). However, in practice the holdout dataset is rarely used only once, and as a result the predictor may not be independent of the holdout set, resulting in overfitting to the holdout set \cite{Reunanen:03,RaoF:08,CawleyT10}. One well-known reason for such dependence is that the holdout data is used to test a large number of predictors and only the best one is reported. If the set of all tested hypotheses is known and independent of the holdout set, then it is easy to account for such multiple testing or use the more sophisticated approach of Ng \cite{Ng97}.

However such static approaches do not apply if the estimates or hypotheses tested on the holdout are chosen adaptively: that is, if the choice of hypotheses depends on previous analyses performed on the dataset. One prominent example in which a holdout set is often adaptively reused is hyperparameter tuning (\eg \cite{DoFN07}).
Similarly, the holdout set in a machine learning competition, such as the famous ImageNet competition, is typically reused many times adaptively.
Other examples include using the holdout set for feature selection, generation of base learners (in aggregation techniques such as boosting and bagging), checking a stopping condition, and analyst-in-the-loop decisions. See \cite{Lanford05blog} for a discussion of several subtle causes of overfitting.

The concrete practical problem we address is how to ensure that the holdout set can be reused to perform validation in the adaptive setting. Towards addressing this problem we also ask the more general question of how one can ensure that the final output of adaptive data analysis generalizes to the underlying data distribution. This line of research was recently initiated by the authors in \cite{DworkFHPRR14:arxiv}, where we focused on the case of estimating expectations of functions from i.i.d.~samples (these are also referred to as statistical queries). They show how to answer a large number of adaptively chosen statistical queries using techniques from differential privacy \cite{DMNS06}\iffull (see Sec.~\ref{sec:related} and Sec.~\ref{sec:dp} for more details)\fi.

\subsection{Our Results}
We propose a simple and general formulation of the problem of preserving statistical validity in adaptive data analysis. We show that the connection between differentially private algorithms and generalization from \cite{DworkFHPRR14:arxiv} can be extended to this more general setting, and show that similar (but sometimes incomparable) guarantees can be obtained from algorithms whose outputs can be described by short strings. We then define a new notion, {\em approximate max-information}, that unifies these two basic techniques and gives a new perspective on the problem. In particular, we give an adaptive composition theorem for max-information, which gives a simple way to obtain generalization guarantees for analyses in which some of the procedures are differentially private and some have short description length outputs.  We apply our techniques to the problem of reusing the holdout set for validation in the adaptive setting.

\subsubsection{A Reusable Holdout}
We describe a simple and general method, together with two specific instantiations, for {\em reusing} a holdout set for validating results while provably avoiding overfitting to the holdout set. The analyst can perform any analysis on the training dataset, but can only access the holdout set via an algorithm that allows the analyst to validate her hypotheses against the holdout set. Crucially, our algorithm prevents overfitting to the holdout set even when the analyst’s hypotheses are chosen adaptively on the basis of the previous responses of our algorithm.

Our first algorithm, referred to as \Tho, derives its guarantees from differential privacy and the results in \cite{DworkFHPRR14:arxiv,NissimS15}. For any function $\phi:\X \rightarrow [0,1]$ given by the analyst, \Tho uses the holdout set to validate that $\phi$ does not overfit to the training set, that is, it checks that the mean value of $\phi$ evaluated on the training set is close to the mean value of $\phi$ evaluated on the distribution $\calP$ from which the data was sampled. The standard approach to such validation would be to compute the mean value of $\phi$ on the holdout set. The use of the holdout set in \Tho differs from the standard use in that it exposes very little information about the mean of $\phi$ on the holdout set: if $\phi$ does not overfit to the training set, then the analyst receives only the confirmation of closeness, that is, just a single bit. On the other hand, if $\phi$ overfits then \Tho returns the mean value of $\phi$ on the training set perturbed by carefully calibrated noise.

Using results from  \cite{DworkFHPRR14:arxiv,NissimS15} we show that for datasets consisting of i.i.d.~samples these modifications provably prevent the analyst from constructing functions that overfit to the holdout set. This ensures correctness of \Tho's responses. Naturally, the specific guarantees depend on the number of samples $n$ in the holdout set. The number of queries that \Tho can answer is exponential in $n$ as long as the number of times that the analyst overfits is at most quadratic in $n$.

Our second algorithm \Sparse is based on the idea that if most of the time the analyst’s procedures generate results that do not overfit, then validating them against the holdout set does not reveal much information about the holdout set. Specifically, the generalization guarantees of this method follow from the observation that the transcript of the interaction between a data analyst and the holdout set can be described concisely. More formally, this method allows the analyst to pick any Boolean function of a dataset $\psi$ (described by an algorithm) and receive back its value on the holdout set. A simple example of such a function would be whether the accuracy of a predictor on the holdout set is at least a certain value $\alpha$. (Unlike in the case of \Tho, here there is no need to assume that the function that measures the accuracy has a bounded range or even Lipschitz, making it qualitatively different from the kinds of results achievable subject to differential privacy). A more involved example of validation would be to run an algorithm on the holdout dataset to select an hypothesis and check if the hypothesis is similar to that obtained on the training set (for any desired notion of similarity). Such validation can be applied to other results of analysis; for example one could check if the variables selected on the holdout set have large overlap with those selected on the training set. An instantiation of the \Sparse algorithm has already been applied to the problem of answering statistical (and more general) queries in the adaptive setting \cite{BassilySSU15}. \iffull We describe the formal guarantees for \Sparse in Section \ref{sec:sparse}.\fi

\iffull
In Section \ref{sec:experiments} we \else We \fi  describe a simple experiment on synthetic data that illustrates the danger of reusing a standard holdout set, and how this issue can be resolved by our reusable holdout. The design of this experiment is inspired by Freedman's classical experiment, which demonstrated the dangers of performing variable selection and regression on the same data \cite{Freedman83}.

\subsection{Generalization in Adaptive Data Analysis}
We view adaptive analysis on the same dataset as an execution of a sequence of steps $\A_1 \rightarrow \A_2 \rightarrow \cdots \rightarrow \A_m$. Each step is described by an algorithm $\A_i$ that takes as input a fixed dataset $S=(x_1,\ldots,x_n)$ drawn from some distribution $\D$ over $\X^n$, which remains unchanged over the course of the analysis. Each algorithm $A_i$ also takes as input the outputs of the previously run algorithms $\A_1$ through $\A_{i-1}$ and produces a value in some range $\Y_i$.
The dependence on previous outputs represents all the adaptive choices that are made at step $i$ of data analysis. For example, depending
on the previous outputs, $\A_i$ can run different types of analysis on $S$. We note that at this level of generality, the algorithms can represent the choices of the data analyst, and need not be explicitly specified. We assume that the analyst uses algorithms which {\it individually} are known to generalize when executed on a fresh dataset sampled independently from a distribution $\D$. We formalize this by assuming that for every fixed value $y_1,\ldots,y_{i-1} \in \Y_1 \times \cdots \times \Y_{i-1}$, with probability at least $1-\beta_i$ over the choice of $S$ according to distribution $\D$, the output of $\A_i$ on inputs $y_1,\ldots,y_{i-1}$ and $S$ has a desired property relative to the data distribution $\D$ (for example has low generalization error). Note that in this assumption $y_1,\ldots,y_{i-1}$ are fixed and independent of the choice of $S$, whereas the analyst will execute $\A_i$ on values $\bY_1,\ldots, \bY_{i-1}$, where $\bY_j = \A_j(S, \bY_{1},\ldots, \bY_{j-1})$. In other words, in the adaptive setup, the algorithm $\A_i$ can depend on the previous outputs, which depend on $S$, and thus the set $S$ given to $\A_i$ is no longer an independently sampled dataset. Such dependence invalidates the generalization guarantees of individual procedures, potentially leading to overfitting.

\para{Differential privacy:}
First, we spell out how the differential privacy based approach from  \cite{DworkFHPRR14:arxiv} can be applied to this more general setting. Specifically, a simple corollary of results in \cite{DworkFHPRR14:arxiv} is that for a dataset consisting of i.i.d.~samples any output of a differentially-private algorithm can be used in subsequent analysis while controlling the risk of overfitting, even beyond the setting of statistical queries studied in \cite{DworkFHPRR14:arxiv}. A key property of differential privacy in this context is that it composes adaptively: namely if each of the algorithms used by the analyst is differentially private, then the whole procedure will be differentially private (albeit with worse privacy parameters). Therefore, one way to avoid overfitting in the adaptive setting is to use algorithms that satisfy (sufficiently strong) guarantees of differential-privacy. \iffull In Section \ref{sec:dp} we describe this result formally.\fi

\para{Description length:}
We then show how description length bounds can be applied in the context of guaranteeing generalization in the presence of adaptivity. If the total length of the outputs of algorithms $\A_1, \ldots, \A_{i-1}$ can be described with $k$  bits then there are at most $2^k$ possible values of the input $y_1,\ldots,y_{i-1}$ to $\A_i$. For each of these individual inputs $\A_i$ generalizes with probability $1-\beta_i$. Taking a union bound over failure probabilities implies generalization with probability at least $1-2^k\beta_i$.
Occam's Razor famously implies that shorter hypotheses have lower generalization error. Our observation is that shorter hypotheses (and the results of analysis more generally) are also better in the adaptive setting since they reveal less about the dataset and lead to better generalization of {\em subsequent} analyses. Note that this result makes no assumptions about the data distribution $\D$. \iffull We provide the formal details in Section \ref{sec:descr}. In Section \ref{sec:median} \else In the full version\fi we also show that description length-based analysis suffices for obtaining an algorithm (albeit not an efficient one) that can answer an exponentially large number of adaptively chosen statistical queries. This provides an alternative proof for one of the results in \cite{DworkFHPRR14:arxiv}.

\para{Approximate max-information:}
Our main technical contribution is the introduction and analysis of a new  information-theoretic measure, which unifies the generalization arguments that come from both differential privacy and description length, and that quantifies how much information has been learned about the data by the analyst. Formally, for jointly distributed random variables $(\bS,\bY)$, the max-information is the maximum of the logarithm of the factor by which uncertainty about $\bS$ is reduced given the value of $\bY$, namely $I_\infty(\bS,\bY) \doteq \log \max \frac{\pr[\bS = S \cond \bY = y]}{\pr[\bS = S]}$, where the maximum is taken over all $S$ in the support of $\bS$ and $y$ in the support $\bY$. \iffull Informally, $\beta$-approximate max-information requires that the logarithm above be bounded with probability at least $1-\beta$ over the choice of $(\bS,\bY)$ (the actual definition is slightly weaker, see Definition \ref{def:maxinfo} for details).\else Approximate max-information is a relaxation of max-information.\fi
In our use, $\bS$ denotes a dataset drawn randomly from the distribution $\D$ and $\bY$ denotes the output of a (possibly randomized) algorithm on $\bS$.
We prove that approximate max-information has the following properties
\begin{itemize}
\item An upper bound on (approximate) max-information gives generalization guarantees.
\item Differentially private algorithms have low max-information for any distribution $\D$ over datasets. A stronger bound holds for approximate max-information on i.i.d.~datasets. These bounds apply only to so-called pure differential privacy (the $\delta=0$ case).
\item Bounds on the description length of the output of an algorithm give bounds on the approximate max-information of the algorithm for any $\D$.
\item Approximate max-information composes adaptively.
\item Approximate max-information is preserved under post-processing.
\end{itemize}
Composition properties of approximate max-information imply that one can easily obtain generalization guarantees for adaptive sequences of algorithms, some of which are differentially private, and others of which have outputs with short description length. These properties also imply that differential privacy can be used to control generalization for any distribution $\D$ over datasets, which extends its generalization guarantees beyond the restriction to datasets drawn i.i.d.~from a fixed distribution, as in  \cite{DworkFHPRR14:arxiv}.

We remark that (pure) differential privacy and description length are otherwise incomparable -- low description length is \emph{not} a sufficient condition for differential privacy, since differential privacy precludes revealing even a small number of bits of information about any single individual in the data set. At the same time differential privacy does not constrain the description length of the output. Bounds on max-information or differential privacy of an algorithm can,
however, be translated to bounds on {\em randomized description length} for a different algorithm with statistically indistinguishable output. Here we say that a randomized algorithm has randomized description length of $k$ if for every fixing of the algorithm's random bits, it has description length of $k$.
Details of these results and additional discussion appear in \iffull Sections \ref{sec:maxinfo} and \ref{sec:max-info2dl}.\else Section \ref{sec:maxinfo} and supplemental material.\fi

\subsection{Related Work}
\label{sec:related}
 This work builds on \cite{DworkFHPRR14:arxiv} where we initiated the formal study of adaptivity in data analysis. The primary focus of \cite{DworkFHPRR14:arxiv} is the problem of answering adaptively chosen statistical queries. The main technique is a strong connection between differential privacy and generalization: differential privacy guarantees that the distribution of outputs does not depend too much on any one of the data samples, and thus, differential privacy gives a strong stability guarantee that behaves well under adaptive data analysis. The link between generalization and approximate differential privacy made in \cite{DworkFHPRR14:arxiv} has been subsequently strengthened, both qualitatively  --- by \cite{BassilySSU15}, who make the connection for a broader range of queries --- and quantitatively, by \cite{NissimS15} and \cite{BassilySSU15}, who give tighter quantitative bounds. These papers, among other results, give methods for accurately answering exponentially (in the dataset size) many adaptively chosen queries, but the algorithms for this task are not efficient. It turns out this is for fundamental reasons -- Hardt and Ullman \cite{HU14} and Steinke and Ullman \cite{SU14} prove that, under cryptographic assumptions, no efficient algorithm can answer more than quadratically many statistical queries chosen adaptively by an adversary who knows the true data distribution.

\iffull
Differential privacy emerged from a line of work \cite{DN03,DworkNi04,BlumDMN05}, culminating in the definition given by
\cite{DMNS06}. There is a very large body of work designing
differentially private algorithms for various data analysis tasks, some of
which we leverage in our applications. See \cite{Dwork11} for a short survey and \cite{DworkR14} for a textbook introduction to differential privacy.
\fi

The classical approach in theoretical machine learning to ensure that
empirical estimates generalize to the underlying distribution is based on the
various notions of complexity of the set of functions output by the algorithm,
most notably the VC dimension\iffull (see e.g.~\cite{Shalev-ShwartzBen-David:2014} for a textbook
introduction)\fi. If one has a sample of data large enough to guarantee
generalization for all functions in some class of bounded complexity, then it
does not matter whether the data analyst chooses functions in this class
adaptively or non-adaptively. Our goal, in contrast, is to prove
generalization bounds \emph{without} making any assumptions about the class
from which the analyst can choose query functions. In this case the adaptive
setting is very different from the non-adaptive setting.

An important line of work~\cite{BousquettE02,MukherjeeNPR06,PoggioRMN04,ShwartzSSS10}
establishes connections between the \emph{stability} of a learning algorithm
and its ability to generalize. Stability is a measure of how much the output of
a learning algorithm is perturbed by changes to its input. It is known that
certain stability notions are necessary and sufficient for generalization.
Unfortunately, the stability notions considered in these prior works do not
compose in the sense that running multiple stable algorithms sequentially and
adaptively may result in a procedure that is not stable. The measure we introduce in this work (max information), like differential privacy, has the strength that it enjoys adaptive composition guarantees. This makes it amenable to reasoning about the generalization properties of adaptively applied sequences of algorithms, while having to analyze only the individual components of these algorithms.
Connections between stability, empirical risk minimization and differential privacy in the context of learnability have been recently explored in \cite{WangLF15}.

Freund  gives an approach to obtaining data-dependent generalization bounds that takes into account the set of statistical queries that a given learning algorithm can produce for the distribution from which the data was sampled \cite{Freund98}. A related approach of Langford and Blum  also allows to obtain data-dependent generalization bounds based on the description length of functions that can be output for a data distribution \cite{LangfordB03}. Unlike our work, these approaches require the knowledge of the structure of the learning algorithm to derive a generalization bound. More importantly, the focus of our framework is on the design of {\em new} algorithms with better generalization properties in the adaptive setting.

Finally, inspired by our work, Blum and Hardt \cite{BlumH15} showed how to reuse the holdout set to maintain an accurate leaderboard in a machine learning competition that allows the participants to submit adaptively chosen models in the process of the competition (such as those organized by Kaggle Inc.). Their analysis also relies on the description length-based technique we used to analyze \Sparse.

%% file: prelims-info-reuse.tex
\iffull
\section{Preliminaries and Basic Techniques}
\else
\subsection{Preliminaries}
\fi
\label{sec:prelims}
In the discussion below $\log$ refers to binary logarithm and $\ln$ refers to the natural logarithm. \iffull For simplicity we restrict our random variables to finite domains (extension of the claims to continuous domains is straightforward using the standard formalism). \fi
For two random variables $\bX$ and $\bY$ over the same domain $\X$ the max-divergence of $\bX$ from $\bY$ is defined as
$$D_\infty (\bX\| \bY) = \log \max_{x \in \X} \frac{\pr[\bX = x]}{\pr[\bY= x]}.$$
$\delta$-approximate max-divergence is defined as
$$D_\infty^\delta (\bX\| \bY) = \log \max_{\cO \subseteq \X,\ \pr[\bX \in \cO]>\delta} \frac{\pr[\bX \in \cO]-\delta}{\pr[\bY \in \cO]}.$$

\iffull
We say that a real-valued function over datasets $f:\X^n \rightarrow \R$ has sensitivity $c$ for all $i\in
[n]$ and $x_1,x_2,\ldots,x_n,x_i'\in\X$,
$f(x_1,\ldots,x_i,\ldots,x_n) - f(x_1,\ldots,x_i',\ldots,x_n) \leq c$.
We review McDiarmid's concentration inequality for functions of low-sensitivity.
\begin{lemma}[McDiarmid's inequality]
\label{lem:mcdiarmid}
Let $\bX_1,X_2, \ldots, \bX_n$ be independent random variables taking values in the set $\X$. Further let $f:\X^n \rar \R$ be a function of sensitivity $c>0$. Then for all $\al > 0$, and $\mu = \E\lb f(\bX_1,\ldots,\bX_n) \rb$, $$\pr\lb f\b(\bX_1,\ldots,\bX_n)- \mu \geq  \alpha \rb \leq \exp\lp\frac{-2\alpha^2}{n \cdot c^2}\rp.$$
\end{lemma}
\fi
For a function $\phi:\X \rightarrow \R$ and a dataset $S=(x_1,\ldots,x_n)$, let $\cE_S[\phi] \doteq \frac1n\sum_{i=1}^n\phi(x_i)$.
\iffull
Note that if the range of $\phi$ is in some interval of length $\alpha$ then $f(S)=\cE_S[\phi]$ has sensitivity $\alpha/n$.
For a distribution $\calP$ over $\X$ and a function $\phi:\X \rar \R$, let $\calP[\phi] \doteq\E_{x\sim \calP}[\phi(x)]$.
\fi
\iffull
\subsection{Differential Privacy}
On an intuitive level, differential privacy hides the data of any single individual.  We are thus interested in pairs of datasets $S,S'$ that differ in a single element, in which case we say $S$ and $S'$ are {\em adjacent}.
\fi
\begin{definition}{\cite{DMNS06,DKMMN06}}
\label{def:dp}
A randomized algorithm $\A$ with domain $\X^n$ for $n > 0$ is $(\eps, \delta)$-differentially private if for all pairs of datasets that differ in a single element $S,S' \in \X^n$:
$D_\infty^\delta( \A(S) \| \A(S') ) \leq \log(e^\eps).$
The case when $\delta = 0$ is sometimes referred to as {\em pure} differential privacy, and in this case we may say simply that $\A$ is $\eps$-differentially private.
\end{definition}
\iffull
Differential privacy is preserved under adaptive composition.
 Adaptive composition of algorithms is a sequential execution of algorithms on the same dataset in which an algorithm at step $i$ can depend on the outputs of previous algorithms. More formally,
let $\A_1, \A_2, \ldots, \A_m$ be a sequence of algorithms. Each algorithm $\A_i$ outputs a value in some range $\Y_i$ and takes as  an input dataset in $\X^n$ as well as a value in $\bar{\Y}_{i-1} \doteq \Y_1 \times \cdots \times \Y_{i-1}$. Adaptive composition of these algorithm is the algorithm that takes as an input a dataset $S \in \X^n$ and executes $\A_1 \rightarrow \A_2 \rightarrow \cdots \rightarrow \A_m$ sequentially with the input to $\A_i$ being $S$ and the outputs $y_1,\ldots,y_{i-1}$ of $\A_1,\ldots,\A_{i-1}$.
Such composition captures the common practice in data analysis of using the outcomes of previous analyses (that is  $y_1,\ldots,y_{i-1}$) to select an algorithm that is executed on $S$.

For an algorithm that in addition to a dataset has other input we say that it is $(\eps,\delta)$-differentially private if it is $(\eps,\delta)$-differentially private for every setting of additional parameters. The basic property of adaptive composition of differentially private algorithms is the following result (\eg \cite{DworkL09}):
\begin{theorem}
\label{thm:easy-composition}
Let $\A_i: \X^n \times \Y_1 \times \cdots \times \Y_{i-1} \rightarrow \Y_i$ be an $(\eps_i,\delta_i)$-differentially private algorithm for $i \in [m]$.  Then the algorithm $\B: \X^n  \rightarrow \Y_m$ obtained by composing $\A_i$'s adaptively is  $(\sum_{i=1}^m\eps_i,\sum_{i=1}^m \delta_i)$-differentially private.
\end{theorem}

A more sophisticated argument yields significant improvement when $\eps < 1$ (\eg \cite{DworkR14}).
\begin{theorem}
\label{thm:composition-advanced}
For all $\eps,\delta, \delta' \geq 0$, the adaptive composition of $m$ arbitrary $(\eps,\delta)$-differentially private algorithms is $(\eps',m\delta+\delta')$-differentially private, where
$$\eps' = \sqrt{2m\ln(1/\delta')}\cdot \eps + m\eps(e^\eps-1).$$
\end{theorem}

Another property of differential privacy important for our applications is preservation of its guarantee under post-processing (\eg \cite[Prop.~2.1]{DworkR14}):
\begin{lemma}
\label{lem:post-process}
If $\mathcal{A}$ is an $(\epsilon,\delta)$-differentially private algorithm with domain $\mathcal{X}^n$ and range $\Y$, and $\mathcal{B}$ is any, possibly randomized, algorithm with domain $\Y$ and range $\Y''$, then the algorithm $\mathcal{B} \circ \mathcal{A}$ with domain $\mathcal{X}^n$ and range $\Y'$ is also $(\epsilon,\delta)$-differentially private.
\end{lemma}
\fi

%% file: dp-info-reuse.tex
\subsection{Generalization via Differential Privacy}
\label{sec:dp}
Generalization in special cases of our general adaptive analysis setting can be obtained directly from results in \cite{DworkFHPRR14:arxiv} and composition properties of differentially private algorithms. For the case of pure differentially private algorithms with general outputs over i.i.d.~datasets, in \cite{DworkFHPRR14:arxiv} we prove the following result.
\begin{theorem}
\label{thm:pure-iid-bound}
Let $\A$ be an $\eps$-differentially private algorithm with range $\Y$
and let $\bS$ be a random variable drawn from a distribution $\calP^n$
 over $\X^n.$ Let $\bY = \A(\bS)$ be the corresponding output distribution.
Assume that for each element $y \in \Y$ there is a subset
$R(y)\subseteq\X^n$ so that $\max_{y\in \Y} \pr[\bS \in R(y)]\le\beta$.
Then, for $\eps \leq \sqrt{\frac{\ln(1/\beta)}{2n}}$ we have
$\Pr[\bS \in R(\bm Y) ] \leq 3\sqrt{\beta}$.
\end{theorem}
An immediate corollary of Thm.~\ref{thm:pure-iid-bound} together with Lemma \ref{lem:mcdiarmid} is that differentially private algorithms that output low-sensitivity functions generalize.
\begin{corollary}
\label{cor:pure-strong bound}
Let $\A$ be an algorithm that outputs a $c$-sensitive function $f:\X^n \rar \R$. Let $\bS$ be a random dataset chosen according to distribution $\calP^n$ over $\X^n$ and let $\bm{f} = \A(\bS)$. If $\A$ is $\tau/(cn)$-differentially private then $\pr[\bm{f}(\bS) - \calP^n[\bm{f}] \geq \tau]\leq 3\exp{(-\tau^2/(c^2 n))}$.
\end{corollary}

By Theorem \ref{thm:easy-composition}, pure differential privacy composes adaptively. Therefore, if in a sequence of algorithms  $\A_1, \A_2, \ldots, \A_m$ algorithm $\A_i$ is $\eps_i$-differentially private for all $i\leq m-1$ then composition of the first $i-1$ algorithms is $\eps'_{i-1}$-differentially private for $\eps'_{i-1} = \left(\sum_{j=1}^{i-1} \eps_j\right)$.
Theorem \ref{thm:pure-iid-bound} can be applied to preserve the generalization guarantees of the last algorithm $\A_m$ (that does not need to be differentially private). For example, assume that for every fixed setting of $\bar{y}_{m-1}$, $\A_m$ has the property that it outputs a hypothesis function $h$ such that, $\pr[\cE_S[L(h)]- \calP[L(h)]  \geq \tau] \leq e^{- n \tau^2/d}$, for some notion of dimension $d$ and a real-valued loss function $L$. Generalization bounds of this type follow from uniform convergence arguments based on various notions of complexity of hypotheses classes such as VC dimension, covering numbers, fat-shattering dimension and Rademacher complexity (see \cite{Shalev-ShwartzBen-David:2014} for examples). Note that, for different settings of $\bar{y}_{m-1}$, different sets of hypotheses and generalization techniques might be used. We define $R(\bar{y}_{m-1})$ be all datasets $S$ for which $\A_m(S,\bar{y}_{m-1})$ outputs $h$ such that $\cE_S[L(h)]- \calP[L(h)] \geq \tau$.
Now if $\eps'_{m-1} \leq \sqrt{\tau^2/(2d)}$, then even for the hypothesis output in the adaptive execution of $\A_m$ on a random i.i.d.~dataset $\bS$ (denoted by $\bm{h}$) we have $\pr\left[\cE_{\bS}[L(\bm{h})] - \calP[L(\bm{h})] \geq \tau\right]\le 3e^{-\tau^2 n/(2d)}.$

For approximate $(\eps,\delta)$-differential privacy, strong preservation of generalization results are currently known only 
for algorithms that output a function over $\X$ of bounded range (for simplicity we use range $[0,1]$) \cite{DworkFHPRR14:arxiv,NissimS15}.
The following result was proved by Nissim and Stemmer \cite{NissimS15} (a weaker statement is also given in \cite[Thm.~10]{DworkFHPRR14:arxiv}).
\begin{theorem}
\label{thm:epsdelta for counts}
Let $\A$ be an $(\eps,\delta)$-differentially private algorithm that outputs a
function from $\X$ to $[0,1]$. For a random variable $\bm S$ distributed according to  $\calP^n$
we let $\bphi = \A(\bS).$ Then for $n \ge 2 \ln(8/\delta)/\eps^2$,
$$\Pr\left[|\calP[\bphi] - \cE_{\bS}[\bphi]| \geq 13\eps\right]\le \frac{2\delta}{\eps}\ln\left(\frac{2}{\eps}\right).$$
\end{theorem}
Many learning algorithms output a hypothesis function that aims to minimize some bounded loss function $L$ as the final output. If algorithms used in all steps of the adaptive data analysis are differentially private and the last step (that is, $\A_m$) outputs a hypothesis $h$, then generalization bounds for the loss of $h$ are implied directly by Theorem \ref{thm:epsdelta for counts}. We remark that this application is different from the example for pure differential privacy above since there we showed preservation of generalization guarantees of arbitrarily complex learning algorithm $\A_m$ which need not be differentially private. In Section \ref{sec:holdout} we give an application of  Theorem \ref{thm:epsdelta for counts} to the reusable holdout problem.

\remove{
 More generally, a request for an estimate of the expectation of a bounded function on $\X$ is referred to as a statistical query in
the context of the well-studied statistical query model~\cite{Kearns93}, and it is known that using statistical queries in place of direct access to data it is possible to implement most standard analyses used on i.i.d.\,data (see \cite{Kearns93,ChuKLYBNO:06} for examples).
Thus, preservation of generalization of such functions allows correctly performing analyses that rely on statistical queries.
} 

%% file: descr-info-reuse.tex
\subsection{Generalization via Description Length}
\label{sec:descr}
Let $\A:\X^n \rightarrow \Y$  and $\B: \X^n \times \Y  \rightarrow \Y'$ be two algorithms. We now give a simple application of bounds on the size of $\Y$ (or, equivalently, the description length of $\A$'s output) to preserving generalization of $\B$. Here generalization can actually refer to any valid or desirable output of $\B$ for a given given dataset $S$ and input $y \in \Y$. Specifically we will use a set $R(y) \subseteq \X^n$ to denote all datasets for which the output of $\B$ on $y$ and $S$ is ``bad" (\eg~overfits). Using a simple union bound we show that the probability (over a random choice of a dataset) of such bad outcome can be bounded.

\begin{theorem}
\label{thm:descr-basic}
Let $\A:\X^n \rightarrow \Y$ be an algorithm and let $\bS$ be a random dataset over $\X^n$. Assume that $R:\Y \rightarrow 2^{\X^n}$ is such that for every $y \in \Y$, $\pr[\bS \in R(y)] \leq \beta$. Then $\pr[\bS \in R(\A(\bS))] \leq |\Y| \cdot \beta$.
\end{theorem}
\begin{proof}
 $$\pr[\bS \in R(\A(\bS))] \leq \sum_{y \in \Y} \pr[\bS \in R(y)]  \leq |\Y| \cdot \beta.$$
\end{proof}

The case of two algorithms implies the general case since description length composes (adaptively).
Namely, let $\A_1, \A_2, \ldots, \A_m$ be a sequence of algorithms such that each algorithm $\A_i$ outputs a value in some range $\Y_i$ and takes as  an input dataset in $\X^n$ as well as a value in $\bar{\Y}_{i-1}$. Then for every $i$, we can view the execution of $\A_1$ through $\A_{i-1}$ as the first algorithm $\bar{\A}_{i-1}$ with an output in $\bar{\Y}_{i-1}$ and $\A_i$ as the second algorithm.  Theorem \ref{thm:descr-basic} implies that if for every setting of $\bar{y}_{i-1}=y_1,\ldots,y_{i-1} \in \bar{\Y}_{i-1}$,  $R(\bar{y}_{i-1}) \subseteq \X^n$ satisfies that $\pr[\bS \in R(\bar{y}_{i-1})] \leq \beta_i$ then $$\pr[\bS \in R(\bar{\A}_{i-1}(\bS))] \leq |\bar{\Y}_{i-1}| \cdot  \beta_i = \prod_{j=1}^{i-1}|\Y_j| \cdot  \beta_i .$$

In Section \ref{sec:max-info2dl} we describe a generalization of description length bounds to randomized algorithms and show that it possesses the same properties. 

%% file: max-info-reuse.tex
\section{Max-Information}
\label{sec:maxinfo}
Consider two algorithms $\A:\X^n \rightarrow \Y$  and $\B: \X^n \times \Y  \rightarrow \Y'$ that are composed adaptively and assume that for every fixed input $y \in \Y$, $\B$ generalizes for all but fraction $\beta$ of datasets. Here we are speaking of generalization informally: our definitions will support any property of input $y \in \Y$ and dataset $S$.
Intuitively, to preserve generalization of $\B$ we want to make sure that the output of $\A$ does not reveal too much information about the dataset $S$. We demonstrate that this intuition can be captured via a notion of {\em max-information} and its relaxation {\em approximate max-information}.

For two random variables $\bX$ and $\bY$ we use $\bX \times \bY$ to denote the random variable obtained by drawing $\bX$ and $\bY$ independently from their probability distributions.
\begin{definition}
\label{def:maxinfo}
Let $\bX$ and $\bY$ be jointly distributed random variables. The
max-information between $\bX$ and $\bY$\iffull, denoted $I_{\infty}(\bX;\bY)$,
 is the minimal
value of $k$ such that for every $x$ in the support of $\bX$ and $y$ in the
support of $\bY$ we have
$\Pr[\bX=x\mid\bY=y] \le 2^k \Pr[\bX=x].$
Alternatively, \else is defined as \fi $I_{\infty}(\bX;\bY) = D_\infty( (\bX,\bY) \| \bX \times \bY )$.
The $\beta$-approximate max-information is defined as $I_{\infty}^\beta(\bX;\bY) = D_\infty^\beta( (\bX,\bY) \| \bX \times \bY )$.
\end{definition}
\iffull
It follows immediately from Bayes' rule that for all $\beta \geq 0$, $I_{\infty}^\beta(\bm{X};\bm Y) = I_{\infty}^\beta(\bm{Y};\bm X)$. Further, $I_{\infty}(\bX;\bY) \leq k$ if and only if for all $x$ in the support of $\bX$, $D_\infty( \bY \cond \bX = x\ \|\ \bY) \leq k$. Clearly, max-information upper bounds the classical notion of mutual information: $I_\infty(\bX;\bY) \geq I(\bX;\bY)$.

\fi
In our use $(\bm{X},\bm{Y})$ is going to be a joint distribution~$(\bS,\A(\bS))$, where $\bS$ is a random $n$-element dataset and $\A$ is a (possibly randomized) algorithm taking a dataset as an input. \iffull If the output of an algorithm on any distribution $\bS$ has  low approximate max-information then we say that the algorithm has low max-information. More formally:\fi
\begin{definition}
We say that an algorithm $\A$ has $\beta$-approximate max-information of $k$ if for every distribution $\cS$ over $n$-element datasets, $I_{\infty}^\beta(\bS;\A(\bS))\leq k$, where $\bS$ is a dataset chosen randomly according to $\cS$. We denote this by $I_{\infty}^\beta(\A,n)\leq k$.
\end{definition}
An alternative way to define the (pure) max-information of an algorithm is using the maximum of the infinity divergence between distributions on two different inputs.
\begin{lemma}
Let $\A$ be an algorithm with domain $\X^n$ and range $\Y$. Then $I_{\infty}(\A,n) = \max_{S,S' \in \X^n} D_\infty(\A(S)\|\A(S'))$.
\end{lemma}
\begin{proof}
\label{lem:max-info-pairwise}
For the first direction let $k=\max_{S,S' \in \X^n} D_\infty(\A(S)\|\A(S'))$. Let $\bS$ be any random variable over $n$-element input datasets for $\A$ and let $\bY$ be the corresponding output distribution
$\bY = \A(\bS)$. We will argue that $I_{\infty}(\bm Y;\bm S)\leq k$, that
$I_{\infty}(\bm S;\bm Y)\leq k$ follows immediately from the Bayes' rule.
For every $y \in \Y$, there must exist a dataset $S_y$ such that $\Pr[\bY=y\mid\bS=S_y]\leq \Pr[\bY=y]$. Now, by our assumption,
for every $S$,  $\Pr[\bY=y\mid\bS=S]\leq 2^k \cdot \Pr[\bY=y\mid\bS=S_y]$. We can conclude that for every $S$ and every $y$, it holds that $\Pr[\bY=y\mid\bS=S]\leq 2^k \Pr[\bY =y]$. This yields $I_{\infty}(\bY;\bS)\leq k$.

For the other direction let $k= I_{\infty}(\A,n)$, let $S,S' \in \X^n$ and $y \in \Y$. For $\alpha \in (0,1)$, let $\bS$ be the random variable equal to $S$ with probability $\alpha$ and to $S'$ with probability $1-\alpha$ and let $\bY = \A(\bS)$. By our assumption,  $I_{\infty}(\bY;\bS) = I_{\infty}(\bS;\bY) \leq k$. This gives
$$\Pr[\bY=y\mid\bS=S]\leq 2^k \Pr[\bY =y] \leq 2^k\left(\alpha \Pr[\bY=y\mid\bS=S] + (1-\alpha)\Pr[\bY=y\mid\bS=S'] \right)$$ and implies
$$\Pr[\bY=y\mid\bS=S]\leq \frac{2^k (1-\alpha)}{1-2^k\alpha} \cdot \Pr[\bY=y\mid\bS=S'].$$
This holds for every $\alpha > 0$ and therefore $$\Pr[\bY=y\mid\bS=S]\leq 2^k  \cdot \Pr[\bY=y\mid\bS=S'].$$
Using this inequality for every $y \in \Y$ we obtain $D_\infty(\A(S)\|\A(S')) \leq k$.
\end{proof}

\iffull \paragraph{Generalization via max-information:}\fi
An immediate corollary of our definition of approximate max-information is that it controls the probability of ``bad events" that can happen as a result of the dependence of $\A(S)$ on $S$.
\begin{theorem}
\label{thm:maxinfo}
Let $\bS$ be a random dataset in $\X^n$ and $\A$ be an algorithm with range $\Y$ such that for some $\beta \geq 0$,
$I_{\infty}^\beta(\bS;\A(\bS)) = k$. Then for any event $\cO \subseteq \X^n \times \Y$,
$$\Pr[(\bS,\A(\bS)) \in \cO] \leq 2^{k} \cdot \Pr[\bS \times \A(\bS) \in \cO] + \beta .$$
In particular,
$\Pr[(\bS,\A(\bS)) \in \cO] \leq 2^{k} \cdot \max_{y \in \Y}\Pr[ (\bS,y) \in \cO] + \beta .$
\end{theorem}

We remark that mutual information between $\bS$ and $\A(\bS)$ would not suffice for ensuring that bad events happen with tiny probability. For example mutual information of $k$ allows $\Pr[(\bS,\A(\bS)) \in \cO]$ to be as high as $k/(2\log(1/\delta))$, where $\delta = \Pr[\bS \times \A(\bS) \in \cO]$.

\iffull
\paragraph{Composition of max-information:}
\else
\fi
Approximate max-information satisfies the following adaptive composition property:
\begin{lemma}
Let $\cA:\cX^n\rightarrow \cY$ be an algorithm such that $I_{\infty}^{\beta_1}(\A,n)\leq k_1$, and let $\cB:\cX^n\times \cY \rightarrow \cZ$ be an algorithm such that for every $y \in \cY$, $\cB(\cdot, y)$ has $\beta_2$-approximate max-information $k_2$. Let $\cC:\cX^n\rightarrow \cZ$ be defined such that $\cC(S) = \cB(\bS, \cA(S))$. Then $I_{\infty}^{\beta_1+\beta_2}(\cC,n)\leq k_1 + k_2$.
\end{lemma}
\iffull
\begin{proof}
Let $\D$ be a distribution over $\X^n$ and $\bS$ be a random dataset sampled from $\D$. By hypothesis, $I^{\beta_1}_\infty(\bS ; \cA(S)) \leq k_1$. Expanding out the definition for all $\cO \subseteq \cX^n \times \cY$:
$$\pr[(\bS, \cA(\bS)) \in \cO] \leq 2^{k_1}\cdot \Pr[\bS \times \cA(\bS) \in \cO] + \beta_1\ .$$
We also have for all $\cQ \subseteq \cX^n \times \cZ$ and for all $y \in \cY$:
$$\pr[(\bS, \cB(\bS, y)) \in \cQ] \leq 2^{k_2}\cdot \Pr[\bS \times \cB(S, y) \in \cQ] + \beta_2\ .$$
For every $\cO \subseteq \cX^n \times \cY$, define
$$\mu(\cO) = \left(\pr[(\bS, \cA(\bS)) \in \cO] -  2^{k_1}\cdot \Pr[\bS \times \cA(\bS) \in \cO]\right)_+ \ .$$
Observe that $\mu(\cO) \leq \beta_1$ for all $\cO \subseteq \cX^n \times \cY$.
For any event $\cQ \subseteq \cX^n \times \cZ$, we have:
\remove{
\begin{eqnarray*}
&&\pr[(\bS, \cC(\bS)) \in \cQ] \\
 &=& \pr[(\bS, \cB(\bS, \cA(\bS))) \in \cQ] \\
&=& \int_{\cX^n \times \cY} \pr[(S, \cB(S, y)) \in \cQ]\cdot \pr[(\bS,\cA(\bS)) \in (dS,dy)] \\
&\leq& \int_{\cX^n \times \cY} \min\left(\left(2^{k_2}\cdot \pr[S \times \cB(S, y) \in \cQ] + \beta_2\right),1\right)\cdot \pr[(\bS,\cA(\bS)) \in (dS,dy)] \\
&\leq& \int_{\cX^n \times \cY} \left(\min\left(2^{k_2}\cdot \pr[S \times \cB(S, y) \in \cQ],1\right) + \beta_2\right)\cdot \pr[(\bS,\cA(\bS)) \in (dS,dy)] \\
&\leq& \int_{\cX^n \times \cY} \left(\min\left(2^{k_2}\cdot \pr[S \times \cB(S, y) \in \cQ],1\right)\cdot \pr[(\bS,\cA(\bS)) \in (dS,dy)]\right)+\beta_2 \\
&\leq& \int_{\cX^n \times \cY} \left(\min\left(2^{k_2}\cdot \pr[S \times \cB(S, y) \in \cQ],1\right)\cdot \left(2^{k_1} \cdot \left(\pr[\bS\times \cA(\bS) \in (dS,dy)]\right) + \mu(dS,dy)    \right)\right)+\beta_2 \\
&\leq& \int_{\cX^n \times \cY} \left(\min\left(2^{k_2}\cdot \pr[S \times \cB(S, y) \in \cQ],1\right)\cdot \left(2^{k_1} \cdot \left(\pr[\bS\times \cA(\bS) \in (dS,dy)]\right)\right)\right) + \int_{\cX^n \times \cY} \mu(dS,dy) + \beta_2 \\
&\leq& \int_{\cX^n \times \cY} \left(\min\left(2^{k_2}\cdot \pr[S \times \cB(S, y) \in \cQ],1\right)\cdot \left(2^{k_1} \cdot \left(\pr[\bS\times \cA(\bS) \in (dS,dy)]\right)\right)\right) + \beta_1 + \beta_2 \\
&\leq& 2^{k_1+k_2}\cdot\int_{\cX^n \times \cY} \left(\pr[S \times \cB(S, y) \in \cQ]\cdot \pr[\bS\times \cA(\bS) \in (dS,dy)]\right) + \beta_1 + \beta_2 \\
&=& 2^{k_1+k_2}\cdot\pr[\bS\times \cB(\bS, \cA(\bS)) \in \cQ] + (\beta_1 + \beta_2)
\end{eqnarray*}
}

\begin{eqnarray*}
&&\pr[(\bS, \cC(\bS)) \in \cQ] \\
 &=& \pr[(\bS, \cB(\bS, \cA(\bS))) \in \cQ] \\
&=& \sum_{S\in \X^n,y\in \cY} \pr[(S, \cB(S, y)) \in \cQ]\cdot \pr[\bS=S, \cA(\bS) = y] \\
&\leq& \sum_{S\in \X^n,y\in \cY} \min\left(\left(2^{k_2}\cdot \pr[S \times \cB(S, y) \in \cQ] + \beta_2\right),1\right)\cdot \pr[\bS=S, \cA(\bS) = y] \\
&\leq& \sum_{S\in \X^n,y\in \cY} \left(\min\left(2^{k_2}\cdot \pr[S \times \cB(S, y) \in \cQ],1\right) + \beta_2\right)\cdot \pr[\bS=S, \cA(\bS) = y] \\
&\leq& \sum_{S\in \X^n,y\in \cY} \min\left(2^{k_2}\cdot \pr[S \times \cB(S, y) \in \cQ],1\right)\cdot \pr[\bS=S, \cA(\bS) = y]+\beta_2 \\
&\leq& \sum_{S\in \X^n,y\in \cY} \min\left(2^{k_2}\cdot \pr[S \times \cB(S, y) \in \cQ],1\right)\cdot \left(2^{k_1} \cdot \pr[\bS = S] \cdot \pr[\cA(\bS) =y] + \mu(S,y)    \right)+\beta_2 \\
&\leq& \sum_{S\in \X^n,y\in \cY}\min\left(2^{k_2}\cdot \pr[S \times \cB(S, y) \in \cQ],1\right)\cdot 2^{k_1} \cdot \pr[\bS = S] \cdot \pr[\cA(\bS) =y] + \sum_{S\in \X^n,y\in \cY} \mu(S,y) + \beta_2 \\
&\leq& \sum_{S\in \X^n,y\in \cY} \min\left(2^{k_2}\cdot \pr[S \times \cB(S, y) \in \cQ],1\right)\cdot 2^{k_1} \cdot\pr[\bS = S] \cdot \pr[\cA(\bS) =y] + \beta_1 + \beta_2 \\
&\leq& 2^{k_1+k_2}\cdot \left(\sum_{S\in \X^n,y\in \cY}\pr[S \times \cB(S, y) \in \cQ]\cdot \pr[\bS = S] \cdot \pr[\cA(\bS) =y]\right) + \beta_1 + \beta_2 \\
&=& 2^{k_1+k_2}\cdot\pr[\bS\times \cB(\bS, \cA(\bS)) \in \cQ] + (\beta_1 + \beta_2)\ .
\end{eqnarray*}
Applying the definition of max-information, we see that equivalently, $I^{\beta_1+\beta_2}_\infty(\bS ; \cC(\bS)) \leq k_1+k_2$, which is what we wanted.
\end{proof}
This lemma can be iteratively applied, which immediately yields the following adaptive composition theorem for max-information:
\begin{theorem}
Consider an arbitrary sequence of algorithms $\cA_1,\ldots,\cA_k$ with ranges $\cY_1,\ldots,\cY_k$ such that for all $i$, $\cA_i : \cX^n \times \cY_1\times \ldots \times \cY_{i-1} \rightarrow \cY_i$ is such that $\cA_i(\cdot, y_1,\ldots,y_{i-1})$ has $\beta_i$-approximate max-information $k_i$ for all choices of $y_1,\ldots,y_{i-1} \in \cY_1\times \ldots \times \cY_{i-1}$. Let the algorithm $\cB:\cX^n\rightarrow \cY_k$ be defined as follows:
\newline
$\cB(S)$:
\begin{enumerate}
 \item \textbf{Let} $y_1 = \cA_1(S)$.
\item \textbf{For} $i = 2$ to $k$: \textbf{Let} $y_i = \cA_i(S, y_1,\ldots,y_{i-1})$
\item \textbf{Output} $y_k$
\end{enumerate}
Then $\cB$ has $(\sum_i \beta_i)$-approximate max-information $(\sum_i k_i)$.
\end{theorem}
\fi
\iffull

\iffull
\paragraph{Post-processing of Max-information:}
Another useful property that (approximate) max-information shares with differential privacy is preservation under post-processing.
The simple proof of this lemma is identical to that for differential privacy (Lemma \ref{lem:post-process}) and hence is omitted.
\begin{lemma}
\label{lem:max-info-post-process}
If $\mathcal{A}$ is an algorithm with domain $\mathcal{X}^n$ and range $\Y$, and $\mathcal{B}$ is any, possibly randomized, algorithm with domain $\Y$ and range $\Y''$, then the algorithm $\mathcal{B} \circ \mathcal{A}$ with domain $\mathcal{X}^n$ and range $\Y'$ satisfies: for every random variable $\bS$ over $\X^n$ and every $\beta \geq 0$, $I_\infty^\beta(\bS;\mathcal{B} \circ \mathcal{A}(\bS)) \leq I_\infty^\beta(\bS;\mathcal{A}(\bS))$.
\end{lemma}

\subsection{Bounds on Max-information}
\else
\para{Bounds on Max-information:}
\fi
\iffull
We now show that the basic approaches based on description length and (pure) differential privacy are captured by approximate max-information.
\subsubsection{Description Length}
\fi
Description length $k$ gives the following bound on max-information.
\begin{theorem}
\label{thm:dl2maxinfo}
Let $\A$ be a randomized algorithm taking as an input an $n$-element dataset and outputting a value in a finite set $\Y$.
 Then for every $\beta > 0$, $I_{\infty}^\beta(\A,n)\leq \log(|\Y|/\beta)$.
\end{theorem}
\iffull
We will use the following simple property of approximate divergence (\eg \cite{DworkR14}) in the proof. For a random variable $\bX$ over $\X$ we denote by $p(\bX)$ the probability distribution associated with $\bX$.
\begin{lemma} \label{lem:to-pointwise}
Let $\bX$ and $\bY$  be two random variables over the same domain $\X$.
If $$\pr_{x \sim p(\bX)} \left[\frac{\pr[\bX = x]}{ \pr[\bY=x]} \geq 2^k\right] \leq \beta$$ then $D_\infty^\beta (\bX\| \bY)\leq k$. \end{lemma}
\begin{proof}[Proof Thm.~\ref{thm:dl2maxinfo}]
Let $\bS$ be any random variable over $n$-element input datasets and let $\bY$ be the corresponding output distribution
$\bY = \A(\bS)$. We prove that for every $\beta > 0$, $I_{\infty}^\beta(\bS;\bY)\leq  \log(|\Y|/\beta)$.

For $y \in \Y$ we say that $y$ is ``bad" if exists $S$ in the support of $\bS$ such that $$\frac{\pr[\bY = y \cond \bS = S]}{\pr[\bY = y]} \geq |\Y|/\beta .$$ Let $B$ denote the set of all ``bad" $y$'s.
From this definition we obtain that for a ``bad" $y$, $\pr[\bY = y] \leq  \beta/|\Y|$ and therefore $\pr[\bY \in B] \leq \beta$.
Let $\B = \X^n \times B$. Then $$\pr[(\bS,\bY) \in \B] = \pr[\bY \in B] \leq \beta .$$
For every $(S,y) \not\in \B$ we have that $$\pr[\bS=S, \bY = y] = \pr[\bY = y \cond \bS=S] \cdot \pr[\bS=S] \leq  \frac{|\Y|}{\beta} \cdot \pr[\bY=y]\cdot  \pr[\bS=S ],$$ and hence
$$ \pr_{(S,y) \sim p(\bS,\bY)} \left[\frac{\pr[\bS=S, \bY = y]}{\pr[\bS=S ] \cdot \pr[\bY=y]} \geq  \frac{|\Y|}{\beta}\right] \leq \beta .$$
This,  by Lemma \ref{lem:to-pointwise}, gives that
$I_\infty^\beta(\bS;\bY) \leq  \log (|\Y|/\beta)$.
\end{proof}

We note that Thms.~\ref{thm:maxinfo} and \ref{thm:dl2maxinfo} give a slightly weaker form of Thm.~\ref{thm:descr-basic}. Defining event $\cO \doteq \{ (S,y) \cond S\in R(y)\}$, the assumptions of Thm.~\ref{thm:descr-basic} imply that  $\Pr[\bS \times \A(\bS) \in \cO] \leq \beta$. For $\beta' = \sqrt{|\Y|\beta}$, by Thm.~\ref{thm:dl2maxinfo}, we have that $I_{\infty}^{\beta'}(\bS;\A(\bS))\leq  \log(|\Y|/\beta')$. Now applying, Thm.~\ref{thm:maxinfo} gives that $\pr[\bS \in R(\A(\bS))] \leq |\Y|/\beta' \cdot \beta + \beta' = 2\sqrt{|\Y| \beta}$.

In Section \ref{sec:max-info2dl} we introduce a closely related notion of {\em randomized description length} and show that it also provides an upper bound on approximate max-information. More interestingly, for this notion a form of the reverse bound can be proved: A bound on (approximate) max-information of an algorithm $\A$ implies a bound on the {\em randomized description length} of the output of a different algorithm with statistically indistinguishable from $\A$ output.
\fi
\iffull
\subsubsection{Differential Privacy}
\label{sec:puredp-to-maxinfo}
We now show that pure differential privacy implies a bound on max information. We start with
\else
Next we prove
\fi
 a simple bound on max-information of differentially private algorithms that applies to all distributions over datasets.
\iffull
In particular, it implies that the differential privacy-based approach can be used beyond the i.i.d.~setting in \cite{DworkFHPRR14:arxiv}.
\fi
\begin{theorem}
\label{thm:dp2maxinfo}
Let $\A$ be an $\epsilon$-differentially private algorithm.  Then $I_{\infty}(\A,n)\leq \log e \cdot \epsilon n$.
\end{theorem}
\iffull
\begin{proof}
Clearly, any two datasets $S$ and $S'$ differ in at most $n$ elements. Therefore,
for every $y$ we have $\Pr[\bY=y\mid\bS=S]\leq e^{\epsilon n}\Pr[\bY =y\mid\bS=S']$
(this is a direct implication of Definition~\ref{def:dp} referred to as group privacy \cite{DworkR14}), or equivalently, $D_\infty(\A(S) \| \A(S'))\leq \log e \cdot \epsilon n$. By Lemma \ref{lem:max-info-pairwise} we obtain the claim.
\end{proof}
\fi
Finally, we prove a stronger bound on approximate max-information for datasets consisting of i.i.d.~samples using the technique from \cite{DworkFHPRR14:arxiv}. \iffull This bound, together with Thm.~\ref{thm:maxinfo}, generalizes Thm.~\ref{thm:pure-iid-bound}.\fi
\begin{theorem}
\label{thm:concentrated-divergence}
Let $\A$ be an $\eps$-differentially private algorithm with range $\Y$.  For a distribution $\cP$ over $\X$,
 let $\bS$ be a random variable drawn from $\cP^n$.
Let $\bY = \A(\bS)$ denote the random variable output by $\A$ on input $\bS$.
Then for any $\beta>0$, $I_\infty^\beta(\bS;\A(\bS)) \leq  \log e (\eps^2 n/2 + \eps\sqrt{n \ln(2/\beta)/2})$.
\end{theorem}
\iffull
\begin{proof} Fix $y \in \Y$. We first observe that by Jensen's inequality,
\[
\E_{S \sim \cP^n} [\ln(\pr[\bY = y \cond \bS = S])] \leq \ln \lp \E_{S \sim
\cP^n} [\pr[\bY=y \cond \bS  = S]] \rp = \ln(\pr[\bY=y]).
\]
Further, by
definition of differential privacy, for two databases $S,S'$ that differ in a
single element, $$\pr[\bY=y \cond \bS = S] \leq e^\eps \cdot \pr[\bY=y \cond
\bS = S'].$$

Now consider the function $g(S) = \ln\lp\frac{\pr[\bY=y \cond \bS = S]}{\pr[\bY=y]}\rp$. By the properties above we have that $\E[g(\bS)] \leq \ln(\pr[\bY=y]) - \ln(\pr[\bY=y]) = 0$ and $|g(S) - g(S')| \leq \eps.$ This, by McDiarmid's inequality (Lemma \ref{lem:mcdiarmid}), implies that for any $t > 0$,
\begin{equation} \label{eq:main-conc}
\pr[g(\bS) \geq t ] \leq e^{-2t^2/(n\eps^2)}.
\end{equation}

For an integer $i \geq 1$, let $t_i \doteq \eps^2 n/2 + \eps\sqrt{n \ln(2^i/\beta)/2}$ and let
$$B_i \doteq \left\{ S \ \left| \ t_i < g(S) \leq t_{i+1}\right. \right\} .$$
Let $$B_y  \doteq \left\{ S \ \left| g(S) > t_1 \right. \right\} = \bigcup_{i\geq 1} B_i. $$

By inequality \eqref{eq:main-conc}, we have that for $i\geq 1$,
$$\pr[g(\bS) > t_i] \leq  \exp\lp -2 \lp\eps\sqrt{n}/2 +  \sqrt{\ln(2^i/\beta)/2}\rp^2\rp .$$

By Bayes' rule, for every $S \in B_i$,
$$\frac{\pr[\bS=S \cond \bY = y]}{\pr[\bS=S]} = \frac{\pr[\bY=y \cond \bS = S]}{\pr[\bY=y]} = \exp(g(S)) \leq \exp(t_{i+1}).$$
Therefore,
\alequn{\pr[\bS \in B_i \cond \bY=y ] &= \sum_{S\in B_i} \pr[\bS =S \cond \bY = y]  \nonumber\\
&\leq \exp(t_{i+i}) \cdot \sum_{S\in B_i} \pr[\bS =S]\nonumber\\
&\leq \exp(t_{i+i}) \cdot \pr[g(\bS) \geq t_i] \\
&= \exp\lp \eps^2 n/2 + \eps\sqrt{n \ln(2^{i+1}/\beta)/2} - 2 \lp\eps\sqrt{n}/2 +  \sqrt{\ln(2^i/\beta)/2}\rp^2\rp  \nonumber\\
&\leq \exp\lp\eps\sqrt{n/2}\lp\sqrt{\ln(2^{i+1}/\beta)} - 2\sqrt{\ln(2^{i}/\beta)}\rp - \ln(2^i/\beta)\rp \\
& < \exp(-\ln(2^i/\beta)) = \beta/2^i .
\label{eq-i-bucket-bound}
}
An immediate implication of this is that
\equn{\pr[\bS \in B_y \cond \bY=y ] = \sum_i \pr[\bS \in B_i \cond \bY=y ] \leq \sum_{i\geq 1} \beta/2^i \leq \beta.}
Let $\B = \{ (S,y) \cond y \in \Y, S \in B_y\}$. Then
\equ{\pr[(\bS,\bY) \in \B] = \pr[(\bS,\bY) \in B_{\bY}] \leq \beta . \label{eq:divergence-prob-bound}}

For every $(S,y) \not\in \B$ we have that $$\pr[\bS=S, \bY = y] = \pr[\bS=S \cond \bY = y] \cdot \pr[\bY=y] \leq \exp(t_1) \cdot \pr[\bS=S ] \cdot \pr[\bY=y] ,$$ and hence by eq.\eqref{eq:divergence-prob-bound} we get that
$$ \pr_{(S,y) \sim p(\bS,\bY)} \left[\frac{\pr[\bS=S, \bY = y]}{\pr[\bS=S ] \cdot \pr[\bY=y]} \geq \exp(t_1) \right] \leq \beta .$$

This,  by Lemma \ref{lem:to-pointwise}, gives that
$$I_\infty^\beta(\bS;\bY) \leq  \log (\exp (t_1)) = \log e (\eps^2 n/2 + \eps\sqrt{n \ln(2/\beta)/2}) .$$
\end{proof}
\fi

\iffull
\paragraph{Applications:}
We give two simple examples of using the bounds on max-information obtained from differential privacy to preserve bounds on generalization error that follow from concentration of measure inequalities. Strong concentration of measure results are at the core of most generalization guarantees in machine learning.
Let $\A$ be an algorithm that outputs a function $f:\X^n \rar \R$ of sensitivity $c$ and define the ``bad event" $\cO_\tau$ is when the empirical estimate of $f$ is more than $\tau$ away from the expectation of $f(\bS)$ for $\bS$ distributed according to some distribution $\D$ over $\X^n$.  Namely,
\begin{equation}\label{eq:empirical-error}
\cO_\tau = \left\{ (S, f) \colon  f(S) - \D[f] \geq \tau \right\},
\end{equation}
where $\D[f]$ denotes $\E_{S \sim \D}[f(S)]$.


By McDiarmid's inequality (Lem.~\ref{lem:mcdiarmid}) we know that, if $\bS$ is distributed according to $\cP^n$ then $\sup_{f:\X^n \rar \R}\Pr[ (\bS,f) \in \cO_\tau]\le \exp(-2 \tau^2 /(c^2n))$. The simpler bound in Thm.~\ref{thm:dp2maxinfo} implies following corollary.
\begin{corollary}
\label{cor:weak bound}
Let $\A$ be an algorithm that outputs a $c$-sensitive function $f:\X^n \rar \R$. Let $\bS$ be a random dataset chosen according to distribution $\calP^n$ over $\X^n$ and let $\bm{f} = \A(\bS)$.
If for $\beta \geq 0$ and $\tau > 0$, $I_{\infty}^\beta(\bS;\bm{f}) \leq \log e \cdot \tau^2/c^2$, then $\pr[\bm{f}(\bS) - \calP^n[\bm{f}] \geq\tau]\leq \exp{(-\tau^2/(c^2 n))} + \beta$. In particular, if $\A$ is $\tau^2/(c^2n^2)$-differentially private then $\pr[\bm{f}(\bS) - \calP^n[\bm{f}] \geq\tau]\leq \exp{(-\tau^2/(c^2 n))}$.
\end{corollary}

Note that for $f(S) = \cE_S[\phi]$, where $\phi:\X \rightarrow [0,1]$ this result requires $\eps = \tau^2$. The stronger bound allows to preserve concentration of measure even when $\eps = \tau/(cn)$ which corresponds to $\tau = \eps$ when $f(S) = \cE_S[\phi]$.
\begin{corollary}
\label{cor:strong bound}
Let $\A$ be an algorithm that outputs a $c$-sensitive function $f:\X^n \rar \R$. Let $\bS$ be a random dataset chosen according to distribution $\calP^n$ over $\X^n$ and let $\bm{f} = \A(\bS)$. If $\A$ is $\tau/(cn)$-differentially private then $\pr[\bm{f}(\bS) - \calP^n[\bm{f}] \geq\tau]\leq \exp{(-3\tau^2/(4c^2 n))}$.
\end{corollary}
\begin{proof}
We apply Theorem \ref{thm:concentrated-divergence} with $\beta = 2\exp{(-\tau^2/(c^2 n))}$ to obtain that
$$I_\infty^\beta(\bS; \bm{f}) \leq \log e \cdot (\eps^2 n/2 + \eps\sqrt{n \ln(2/\beta)/2})) \leq  \log e \cdot ( \tau^2/(c^2 n)/2 + \tau^2/(c^2 n)/\sqrt{2}). $$
Applying Thm.~\ref{thm:maxinfo} to McDiarmid's inequality we obtain that
\alequn{\pr[\bm{f}(\bS) - \calP^n[\bm{f}] \geq\tau]& \leq   \exp{((1/2+1/\sqrt{2}) \tau^2/(c^2 n))} \cdot \exp{(-2\tau^2/(c^2 n))} +  2\exp{(-\tau^2/(c^2 n))} \\
&\leq \exp{(-3\tau^2/(4c^2 n))},
}
where the last inequality holds when $\tau^2/(c^2 n)$ is larger than a fixed constant.
\end{proof}
\else

One way to apply a bound on max-information is to start with a concentration of measure result which ensures that the estimate of predictor's accuracy is correct with high probability when the predictor is chosen independently of the samples. For example for a loss function with range $[0,1]$, Hoeffding's bound implies that for a dataset consisting of i.i.d.~samples the empirical estimate is not within $\tau$ of the true accuracy with probability $\leq2 e^{-2\tau^2 n}$. Now, given a bound of $\log e \cdot \tau^2 n$ on $\beta$-approximate information of the algorithm that produces the estimator, Thm.~\ref{thm:maxinfo} implies that the produced estimate is not within $\tau$ of the true accuracy with probability $\leq 2^{\log e \cdot \tau^2n } \cdot 2e^{-2\tau^2 n}  +\beta \leq 2e^{-\tau^2 n} +\beta$. Thm.~\ref{thm:dp2maxinfo} implies that any $\tau^2$-differentially private algorithm has max-information of at most $\log e \cdot \tau^2 n$. For a dataset consisting of i.i.d.~samples Thm.~\ref{thm:concentrated-divergence} implies that a $\tau$-differentially private algorithm has $\beta$-approximate max-information of $1.25 \log e \cdot \tau^2n$ for $\beta = 2e^{-\tau^2 n}$. A more detailed and formal example appears in the supplemental material.

\fi

%% file: holdout-info-reuse.tex
\section{Reusable Holdout}
\label{sec:holdout}
We describe two simple algorithms that enable validation of analyst's queries in the adaptive setting.
\iffull
\subsection{\Tho}
\else
\paragraph{Thresholdout:}
\fi
\label{sec:holdout-tho}
Our first algorithm $\Tho$ follows the approach in \cite{DworkFHPRR14:arxiv} where differentially private algorithms are used to answer adaptively chosen statistical queries. This approach can also be applied to any low-sensitivity functions
\iffull\footnote{Guarantees based on pure differential privacy follow from the same analysis. Proving generalization guarantees for low-sensitivity queries based on approximate differential privacy requires a modification of \Tho using techniques in \cite{BassilySSU15}.}
\fi
of the dataset but for simplicity we present the results for statistical queries. Here we address an easier problem in which the analyst's queries only need to be answered when they overfit. Also, unlike in \cite{DworkFHPRR14:arxiv}, the analyst has full access to the training set and  the holdout algorithm only prevents overfitting to holdout dataset. As a result, unlike in the general query answering setting, our algorithm can efficiently validate an exponential in $n$ number of queries as long as a relatively small number of them overfit.

\Tho is given access to the training dataset $S_t$ and holdout dataset $S_h$ and a budget limit $B$. It allows any query of the form $\phi:\X \rar [0,1]$ and its goal is to provide an estimate of $\calP[\phi]$. To achieve this the algorithm gives an estimate of $\cE_{S_h}[\phi]$ in a way that prevents overfitting of functions generated by the analyst to the holdout set. In other words, responses of \Tho are designed to ensure that, with high probability, $\cE_{S_h}[\phi]$ is close to $\calP[\phi]$ and hence an estimate of $\cE_{S_h}[\phi]$ gives an estimate of the true expectation $\calP[\phi]$.
Given a function $\phi$, \Tho first checks if the difference between  the average value of $\phi$
on the training set $S_t$ (or $\cE_{S_t}[\phi]$) and the average
value of $\phi$ on the holdout set $S_h$ (or $\cE_{S_h}[\phi]$) is
below a certain threshold~$T+\eta$. Here, $T$ is a fixed number such as
$0.01$ and $\eta$ is a Laplace noise variable whose standard deviation needs to be chosen depending on the desired guarantees (The Laplace distribution is a symmetric exponential distribution.)  If
the difference is below the threshold, then the algorithm returns
$\cE_{S_t}[\phi]$. If the difference is above the threshold, then the algorithm returns $\cE_{S_h}[\phi]+\xi$ for another Laplacian noise variable $\xi$.
Each time the difference is above threshold the ``overfitting" budget $B$ is reduced by one. Once it is exhausted, \Tho stops answering queries.
\iffull
In Fig.~\ref{fig:Tho} we provide the pseudocode of \Tho.
\else
We provide the pseudocode of \Tho below.
\fi
\begin{figure}[h]
\begin{boxedminipage}{\textwidth}
\iffull \textbf{Algorithm} \Tho \fi
\textbf{Input:} Training set $S_t,$ holdout set $S_h,$ threshold $T,$
noise rate $\sigma$, 
 budget $B$
\begin{enumerate}
    \item sample $\gamma \sim \mathrm{Lap}(2\cdot\sigma)$; $\hat{T} \leftarrow T + \gamma$
    \item \textbf{For} each query $\phi$ \textbf{do}
    \begin{enumerate}
        \item \textbf{if} $B<1$ output ``$\bot$''
        \item \textbf{else}
        \begin{enumerate}
                \item sample $\eta \sim \mathrm{Lap}(4\cdot\sigma)$
                \item \textbf{if} $|\cE_{S_h}[\phi]-\cE_{S_t}[\phi]|>\hat{T}+\eta$
                \begin{enumerate}
                    \item sample $\xi \sim \mathrm{Lap}(\sigma)$, $\gamma \sim \mathrm{Lap}(2\cdot\sigma)$
                    \item $B\leftarrow B-1$ and $\hat{T} \leftarrow T + \gamma$
                    \item output $\cE_{S_h}[\phi] +\xi$
                \end{enumerate}
                \item \textbf{else} output $\cE_{S_t}[\phi].$
        \end{enumerate}
   \end{enumerate}
\end{enumerate}
\end{boxedminipage}
\iffull \caption{The details of \Tho algorithm}\fi
\label{fig:Tho}
\end{figure}

We now establish the formal generalization guarantees that \Tho enjoys.
\iffull
 As the first step we state what privacy parameters are achieved by \Tho.

\begin{lemma}\label{lem:privacy-sparse-vector}
\Tho satisfies
$(2B/(\sigma n),0)$-differential privacy. $\Tho$ also satisfies
$(\sqrt{32B\ln(2/\delta)}/(\sigma n),\delta)$-differential privacy for any $\delta > 0$. 
\end{lemma}
\begin{proof}
\Tho is an instantiation of a basic tool from differential privacy, the ``Sparse Vector Algorithm'' (\cite[Algorithm 2]{DworkR14}), together with the Laplace mechanism (\cite[Defn.~3.3]{DworkR14}). The sparse vector algorithm takes as input a sequence of $c$ sensitivity $1/n$ queries\footnote{In fact, the theorems for the Sparse Vector algorithm in Dwork and Roth are stated for sensitivity 1 queries -- we use them for sensitivity $1/n$ queries of the form $\cE_{S_h}[\phi]$, which results in all of the bounds being scaled down by a factor of $n$.} (here $c = B$, the budget), and for each query, attempts to determine whether the value of the query, evaluated on the private dataset, is above a fixed threshold $T$ or below it. In our instantiation, the holdout set $S_h$ is the private data set, and each function $\phi$ corresponds to the following query evaluated on $S_h$: $f_\phi(S_h) := |\cE_{S_h}[\phi]-\cE_{S_t}[\phi]|$. (Note that the training set $S_t$ is viewed as part of the definition of the query). $\Tho$ then is equivalent to the following procedure: we run the sparse vector algorithm \cite[Algorithm 2]{DworkR14} with $c = B$, queries $f_\phi$ for each function $\phi$, and noise rate $2\sigma$. Whenever an above-threshold query is reported by the sparse vector algorithm, we release its value using the Laplace mechanism \cite[Defn.~3.3]{DworkR14} with noise rate $\sigma$ (this is what occurs every time \Tho answers by outputting $\cE_{S_h}[\phi] +\xi$). By the privacy guarantee of the sparse vector algorithm (\cite[Thm.~3.25]{DworkR14}), the sparse vector portion of $\Tho$ satisfies $(B/(\sigma n),0)$-differential privacy, and simultaneously satisfies $(\frac{\sqrt{8B\ln(2/\delta)}}{\sigma n},\delta/2)$-differential privacy. The Laplace mechanism portion of \Tho satisfies $(B/(\sigma  n),0)$-differential privacy by \cite[Thm.~3.6]{DworkR14}, and simultaneously satisfies $(\frac{\sqrt{8B\ln(2/\delta)}}{\sigma n},\delta/2)$-differential privacy by \cite[Thm.~3.6]{DworkR14} and \cite[Cor.~3.21]{DworkR14}. Finally, the composition of two mechanisms, the first of which is $(\epsilon_1,\delta_1)$-differentially private, and the second of which is $(\epsilon_2,\delta_2)$-differentially private is itself $(\epsilon_1+\epsilon_2,\delta_1+\delta_2)$-differentially private (Thm.~\ref{thm:easy-composition}). Adding the privacy parameters of the Sparse Vector portion of \Tho and the Laplace mechanism portion of \Tho yield the parameters of our theorem.
\end{proof}
We note that tighter privacy parameters are possible (e.g. by invoking the parameters and guarantees of the algorithm ``NumericSparse'' (\cite[Algorithm 3]{DworkR14}), which already combines the Laplace addition step) -- we chose simpler parameters for clarity.

Note the seeming discrepancy between the guarantee provided by \Tho and generalization guarantees in Theorem \ref{thm:epsdelta for counts} and Corollary \ref{cor:pure-strong bound}: while Theorem \ref{thm:epsdelta for counts} promises generalization bounds for functions that are generated by a differentially private algorithm, here we allow an arbitrary data analyst to generate query functions in any way she chooses, with access to the training set and differentially private estimates of the means of her functions on the holdout set. The connection comes from preservation of differential privacy guarantee under post-processing (Lem.~\ref{lem:post-process}).

We can now quantify the number of samples necessary to achieve generalization error $\tau$ with probability at least $1-\beta$.
\begin{lemma}
\label{lem:tho-generalizes}
Let $\tau,\beta,T,B>0$. Let $\bS$ denote the holdout dataset of size $n$ drawn i.i.d.~from a distribution~$\calP$.
Consider an algorithm that is given access to $S_t$ and adaptively chooses
functions $\bphi_1,\dots,\bphi_m: \X\rightarrow [0,1]$ while interacting with \Tho which is given
datasets $\bS$, $S_t$ and parameters  $\sigma,B,T$. If $$n \geq n_0(B,\sigma,\tau,\beta) \doteq \max\{2B/(\sigma\tau), \ln(6/\beta)/\tau^2\}$$ or $$n \geq n_1(B,\sigma,\tau,\beta)  \doteq \frac{80 \cdot \sqrt{B \ln(1/(\tau\beta))}}{\tau \sigma}$$ then for every $i \in [m]$,
$\Pr\left[|\calP[\bphi_i] - \cE_{\bS}[\bphi_i]| \geq \tau\right]\le \beta$.
\end{lemma}
\begin{proof}
Consider the first guarantee of Lemma~\ref{lem:privacy-sparse-vector}.
In order to achieve generalization error $\tau$ via Corollary~\ref{cor:pure-strong bound} (i.e. in order to guarantee that for every function $\phi$ we have:
$\Pr\left[|\calP[\bphi_i] - \cE_{\bS}[\bphi_i]| \geq \tau\right]\le 6e^{-\tau^2 n}$)
we need to have $n$ large enough to achieve $(\eps,0)$-differential privacy for $\eps = \tau$. To achieve this it suffices to have $n \geq 2B/(\sigma\tau)$. By ensuring that $n \geq \ln(6/\beta)/\tau^2$ we also have that $6e^{-\tau^2 n} \leq \beta$.

We can also make use of the second guarantee in Lemma~\ref{lem:privacy-sparse-vector} together with the results of Nissim and Stemmer \cite{NissimS15} (Thm.~\ref{thm:epsdelta for counts}). In order to achieve
generalization error $\tau$ with probability $1-\beta$ (i.e. in order to guarantee for every function $\phi$ we have:
 $\Pr\left[|\calP[\bphi_i] - \cE_{\bS}[\bphi_i]| \geq \tau\right]\le \beta$), we can apply  Thm.~\ref{thm:epsdelta for counts} by setting $\epsilon = \sqrt{32B\ln(2/\delta)}/(\sigma n)=\tau/13$ and $\delta = \frac{\beta\tau}{26\ln{(26/\tau)}}$. We can obtain these privacy parameters from Lemma \ref{lem:privacy-sparse-vector} by choosing any $n \geq  \frac{80 \cdot \sqrt{B \ln(1/(\tau\beta))}}{\tau \sigma}$ (for sufficiently small $\beta$ and $\tau$).
We remark that a somewhat worse bound of $n_1(B,\sigma,\tau,\beta) = \frac{\sqrt{2048\ln(8/\beta)}}{\tau^{3/2} \sigma}$ follows by setting $\epsilon = \tau/4$ and $\delta = (\beta/8)^{4/\tau}$ in \cite[Thm.~10]{DworkFHPRR14:arxiv}.
\end{proof}
\remove{
\begin{equation}
B_0(n,\sigma,\tau) = \tau\sigma n/2.
\end{equation}
\begin{equation}
B_1(n,\sigma,\tau,\beta) = \frac{C \tau^2 \sigma^2 n^2}{\ln(1/(\tau\beta))}\, ,
\end{equation}
}
Both settings lead to small generalization error and so we can pick whichever
gives the larger bound. The first bound has grows linearly with $B$ but is simpler can be easily extended to other distributions over datasets and to low-sensitivity functions. The second bound has quadratically better dependence on $B$ at the expense of a slightly worse dependence on $\tau$. 
\fi
We can now apply our main results to get a generalization bound for the entire execution of \Tho.
\remove{
(This is very similar to the accuracy guarantee for the sparse vector algorithm \cite[Thm.~3.26]{DworkR14}, but the version we derive here states the probability of error per query rather than the probability of error in the worst case over all $m$ queries):

\begin{lemma}
\label{lem:strong holdout}
Let $T\ge0,\tau>0$. Let $\bS$ be a holdout dataset of size $n$ drawn i.i.d.~from a distribution~$\calP$ and and $S_t$ be any additional dataset over $\X$.
Consider an algorithm that is given access to $S_t$ and adaptively chooses
functions $\bphi_1,\dots,\bphi_m$ while interacting with \Tho which is given
datasets $\bS$, $S_t$ tolerance $\tau$ and threshold $T$, and returns an answer $\ba_i$ on function~$\bphi_i:\X\rightarrow [0,1].$
\iffull
If we set $B=B_0(n,\sigma,\tau)$,
then, for all $i\in\{1,\dots,m\}$ and $t>0$,

\[
\Pr\left[\ba_i\neq \bot \ \& \ \left|\ba_i-\calP[\bphi_i]\right|>T + t \sigma + \tau\right]
\le 6e^{-\tau^2 n} + B e^{-t/8}\,.
\]
If we set $B=B_1(n,\sigma,\tau,\beta)$
\else
If we set $B=\frac{\tau^2 \sigma^2|S_h|^2}{6000 \cdot \ln(1/(\tau\beta))}$,
\fi
then for all $i\in\{1,\dots,m\}$ and all $t>0$,
\[
\Pr\left[\ba_i\neq \bot \ \& \ \left|\ba_i-\calP[\bphi_i]\right|>T +  t \sigma + \tau\right]
\le \beta + B e^{-t/8}\,.
\]
\end{lemma}
\iffull
\begin{proof}
There are two types of error we need to control: the deviation between $\ba_i$ and the average value of $\bphi_i$ on the holdout set $\cE_{\bS}[\bphi_i]$, and the deviation between the average value of $\bphi_i$ on the holdout set and the expectation of $\phi_i$ on the underlying distribution,  $\calP[\bphi_i]$.
Namely,
\alequn{\Pr[\ba_i\neq \bot \ \& \ |\ba_i - \calP[\bphi_i]| \geq T+ t \sigma + \tau] & \leq \Pr\left[\ba_i\neq \bot \ \& \ |\ba_i - \cE_{\bS}[\bphi_i]| \geq T+ t\sigma \right] + \Pr\left[|\calP[\bphi_i] - \cE_{\bS}[\bphi_i]| > \tau\right]
}
For the second term, we know that when we set $B = B_0(n,\sigma,\tau)$, by Corollary~\ref{cor:pure-strong bound} we have:
$$\Pr\left[|\calP[\bphi_i] - \cE_{\bS}[\bphi_i]| > \tau\right]\le 6e^{-\tau^2 n}. $$
Similarly, if we set $B = B_1(n,\sigma,\tau,\beta)$, we can bound the second term using Theorem \ref{thm:epsdelta for counts}:
$$\Pr\left[|\calP[\bphi_i] - \cE_{\bS}[\bphi_i]| > \tau\right]\le \beta .$$
Therefore to obtain the claim we only need to bound the first term by $e^{-t/8}$.
There are two cases, depending on whether \Tho answers query $\bphi_i$ by returning $\ba_i = \cE_{\bS}[\bphi_i] + \xi$ or by returning $\ba_i = \cE_{S_t}[\bphi_i]$. First, consider queries whose answers are returned using the former condition. Under this condition the first term, $|\ba_i - \cE_{\bS}[\bphi_i]| = \xi$ by definition of the algorithm, where  $\xi\sim\mathrm{Lap}(\sigma)$. By properties of the Laplace distribution, we know that $\Pr[|\xi| \geq t\cdot\sigma] = e^{-t/2}$. In particular, the desired bound is satisfied.

Next, we consider the second case, those queries whose answers are returned using $\cE_{S_t}[\bphi]$. By definition of the algorithm, we have $|\ba_i - \cE_{\bS}[\bphi_i]| \leq \hat{\bm T}_i + \eta$, where $\eta \sim \mathrm{Lap}(4\cdot \sigma)$ and $\hat{\bm T}_i$ denotes the value of $\hat{T}$ at step $i$.
Further we know that $\hat{\bm T}_i = T+\gamma$ and $\gamma$ is randomized at most $B$ times. Therefore, $\Pr[|T-\hat{\bm T}_i|\geq t \sigma/2 ] \leq B \cdot e^{-t/8}$.
Hence, by properties of the Laplace distribution, we have:
$$\Pr[|\ba_i - \cE_{\bS}[\bphi_i]| \geq T + t \cdot \sigma/2] \leq \Pr[|\ba_i - \cE_{\bS}[\bphi_i]| \geq \hat{\bm T}_i + \eta$$

$$e^{-t/8}.$$
The theorem results from combining the two cases.
\end{proof}
\fi
\iffull

Note that the above theorem gives an accuracy bound for a single query. We now  instantiate Theorem \ref{thm:strong holdout} for a sequence of queries:
}

\begin{theorem}
\label{thm:strong holdout}
Let $\beta,\tau>0$ and $m \geq B > 0$. We set $T = 3\tau/4$ and $\sigma = \tau/(96\ln(4m/\beta))$. Let $\bS$ denote a holdout dataset of size $n$ drawn i.i.d.~from a distribution~$\calP$ and $S_t$ be any additional dataset over $\X$. Consider an algorithm that is given access to $S_t$ and adaptively chooses
functions $\bphi_1,\dots,\bphi_m$ while interacting with \Tho which is given
datasets $\bS$, $S_t$ and values $\sigma,B,T$. For every $i\in [m]$, let $\ba_i$ denote the answer of \Tho on function~$\bphi_i:\X\rightarrow [0,1]$.
Further, for every $i\in [m]$, we define the counter of overfitting
$$\bZ_i \doteq \left|\left\{j \leq i : |\calP[\bphi_j]-\cE_{S_t}[\bphi_j]| > \tau/2\right\}\right|.$$
Then \[
\Pr\left[\exists i \in [m], \bZ_i < B \ \& \ \left|\ba_i-\calP[\bphi_i]\right|\geq \tau\right]
\le \beta\
\]
whenever
$n \geq \min\{n_0(B,\sigma,\tau/8,\beta/(2m)),\ n_1(B,\sigma,\tau/8,\beta/(2m))\}= O\left(\frac{\ln(m/\beta)}{\tau^2}\right) \cdot \min\{B, \sqrt{B \ln(\ln(m/\beta)/\tau)}\}$.
\end{theorem}
\begin{proof}
There are two types of error we need to control: the deviation between $\ba_i$ and the average value of $\bphi_i$ on the holdout set $\cE_{\bS}[\bphi_i]$, and the deviation between the average value of $\bphi_i$ on the holdout set and the expectation of $\phi_i$ on the underlying distribution,  $\calP[\bphi_i]$. Specifically, we decompose the error as
\equ{\Pr[\ba_i\neq \bot \ \& \ |\ba_i - \calP[\bphi_i]| \geq \tau]  \leq \Pr\left[\ba_i\neq \bot \ \& \ |\ba_i - \cE_{\bS}[\bphi_i]| \geq 7\tau/8 \right] + \Pr\left[|\calP[\bphi_i] - \cE_{\bS}[\bphi_i]| \geq \tau/8\right]. \label{eq:two-error}
}
To control the first term we need to bound the values of noise variables used by $\Tho$. For the second term we will use the generalization properties of \Tho given in Lemma \ref{lem:tho-generalizes}.

We now deal with the errors introduced by the noise variables. For $i\in [m]$, let $\bm\eta_i$, $\bm\xi_i$ and $\bm\gamma_i$ denote the random variables $\eta,\xi$ and $\gamma$, respectively, at step $i$ of the execution of $\Tho$. We first note that each of these variables is chosen from Laplace distribution at most $m$ times. By properties of the Laplace distribution with parameter $4\sigma$, we know that for every $t >0$, $\Pr[|\bm\eta_i| \geq t\cdot 4\sigma] = e^{-t/2}$. Therefore for $t=2\ln(4m/\beta)$ we obtain
$$\Pr[|\bm\eta_i| \geq 2\ln(4m/\beta)\cdot 4 \sigma] \leq e^{-t/2} = \frac{\beta}{4m} .$$
By the definition of $\sigma$, $8\ln(4m/\beta)\cdot \sigma = \tau/12$.  Applying the union bound we obtain that
$$\Pr[\exists i,\ |\bm\eta_i| \geq \tau/12] \leq \beta/4 ,$$
where by $\exists i$ we refer to $\exists i\in [m]$ for brevity.
Similarly, $\bm\xi_i$ and $\bm\gamma_i$ are obtained by sampling from the Laplace distribution and each is re-randomized at most $B$ times. Therefore
$$\Pr[\exists i,\ |\bm\gamma_i| \geq \tau/24] \leq B \cdot \beta/4m \leq \beta/4$$ and
$$\Pr[\exists i,\ |\bm\xi_i| \geq \tau/48] \leq B \cdot \beta/4m \leq \beta/4.$$

For answers that are different from $\bot$, we can now bound the first term of Equation \eqref{eq:two-error} by considering two cases, depending on whether \Tho answers query $\bphi_i$ by returning $\ba_i = \cE_{\bS}[\bphi_i] + \bm\xi_i$ or by returning $\ba_i = \cE_{S_t}[\bphi_i]$. First, consider queries whose answers are returned using the former condition. Under this condition, $|\ba_i - \cE_{\bS}[\bphi_i]| = |\bm\xi_i|$.  Next, we consider the second case, those queries whose answers are returned using $\cE_{\bS_t}[\bphi_i]$. By definition of the algorithm, we have $$|\ba_i - \cE_{\bS}[\bphi_i]| = |\cE_{S_t}[\bphi_i] - \cE_{\bS}[\bphi_i]| \leq T + \bm\gamma_i +\bm\eta_i \leq 3\tau/4 + |\bm\gamma_i| +|\bm\eta_i| .$$
Combining these two cases implies that
$$\Pr\left[\exists i,\ \ba_i\neq \bot \ \& \ |\ba_i - \cE_{\bS}[\bphi_i]| \geq 7\tau/8 \right] \leq \max\{\Pr\left[\exists i,\ |\bm\xi_i| \geq 7\tau/8\right],\ \Pr\left[\exists i,\  |\bm\gamma_i| +|\bm\eta_i| \geq \tau/8\right]\}.$$
Noting that $\tau/24+\tau/12 = \tau/8$ and applying our bound on variables $\bm\eta_i$, $\bm\xi_i$ and $\bm\gamma_i$ we get
\equ{\Pr\left[\exists i,\ \ba_i\neq \bot \ \& \ |\ba_i - \cE_{\bS}[\bphi_i]| \geq 7\tau/8\right]\le \beta/2 .\label{eq:bound-noise}}

By Lemma \ref{lem:tho-generalizes}, for $n \geq \min\{n_0(B,\sigma,\tau/8,\beta/2m),n_1(B,\sigma,\tau/8,\beta/2m)\}$,
$$\Pr\left[|\calP[\bphi_i] - \cE_{\bS}[\bphi_i]| \geq \tau/8\right]\leq \beta/2m .$$
Applying the union bound we obtain
$$\Pr\left[\exists i,\ |\calP[\bphi_i] - \cE_{\bS}[\bphi_i]| \geq \tau/8\right]\le \beta/2 .$$
Combining this with Equation \eqref{eq:bound-noise} and using in Equation \eqref{eq:two-error} we get that
$$\Pr[\exists i,\ \ba_i\neq \bot \ \& \ |\ba_i - \calP[\bphi_i]| \geq \tau] \leq \beta .$$

To finish the proof we show that under the conditions on the noise variables and generalization error used above, we have that if $\bZ_i < B$ then $\ba_i \neq \bot$. To see this, observe that for every $j \leq i$ that reduces \Tho's budget, we have
\alequn{|\calP[\bphi_j]-\cE_{S_t}[\bphi_j]| &\geq |\cE_{\bS}[\bphi_j]- \cE_{S_t}[\bphi_j]| - |\calP[\bphi_j]-  \cE_{\bS}[\bphi_j]|\\
& \geq |T  + \bm\gamma_j +\bm\eta_j| - |\calP[\bphi_j] -  \cE_{\bS}[\bphi_j]| \\
& \geq T  - |\bm\gamma_j|  - |\bm\eta_j| - |\calP[\bphi_j] -  \cE_{\bS}[\bphi_j]|.} This means that for every $j \leq i$ that reduces the budget we have $|\calP[\bphi_j]-\cE_{S_t}[\bphi_j]| \geq 3\tau/4 - \tau/24-\tau/12-\tau/8 = \tau/2$ and hence (when the conditions on the noise variables and generalization error are satisfied) for every $i$, if $\bZ_i < B$ then \Tho's budget is not yet exhausted and $\ba_i \neq \bot$. We can therefore conclude that $$\Pr[\exists i,\bZ_i < B \ \& \ |\ba_i - \calP[\bphi_i]| \geq \tau] \leq \beta .$$
\end{proof}
Note that in the final bound on $n$, the term $O\left(\frac{\ln(m/\beta)}{\tau^2}\right)$ is equal (up to a constant factor) to the number of samples that are necessary to answer $m$ {\em non-adaptively} chosen queries with tolerance $\tau$ and confidence $1-\beta$. In particular, as in the non-adaptive setting, achievable tolerance $\tau$ scales as $1/\sqrt{n}$ (up to the logarithmic factor). Further, this bound allows $m$ to be exponentially large in $n$ as long as $B$ grows sub-quadratically in $n$  (that is, $B \leq n^{2-c}$ for a constant $c > 0$).
\remove{
We typically choose the
parameters $T$ and $\tau$ so that $T+\tau$ is reasonably small, e.g., $0.05.$
Increasing $T$ relative to $\tau$ reduces the number of functions with above threshold
response, but increases the generalization error. Decreasing $T$ has the
opposite effect and even $T=0$ leads to a non-trivial guarantee.}
\begin{remark}
In Thm.~\ref{thm:strong holdout} $S_h$ is used solely to provide a candidate estimate of the expectation of each query function. The theorem holds for any other way to provide such estimates. In addition, the one-sided version of the algorithm can be used when catching only the one-sided error is necessary. For example, in many cases overfitting is problematic only if the training error estimate is larger than the true error. This is achieved by using the condition $\cE_{S_h}[\phi]-\cE_{S_t}[\phi]>\hat{T}+\eta$ to detect overfitting. In this case only one-sided errors will be caught by \Tho and only one-sided overfitting will decrease the budget.
\end{remark}
\fi

\iffull
\subsection{\Sparse}
\else
\paragraph{SparseValidate:}
\fi
\label{sec:sparse}
We now present a general algorithm for validation on the holdout set that can validate many arbitrary queries as long as few of them fail the validation.
\iffull The algorithm which we refer to as \Sparse only reveals information about the holdout set when validation fails and therefore we use bounds based on description length to analyze its generalization guarantees.

\fi
More formally, our algorithm allows the analyst to pick any Boolean function of a dataset $\psi$ (or even any algorithm that outputs a single bit) and provides back the value of $\psi$ on the holdout set $\psi(S_h)$. \Sparse has a budget $m$ for the total number of queries that can be asked and budget $B$ for the number of queries that returned $1$. Once either of the budgets is exhausted, no additional answers are given. We now give a general description of the guarantees of \Sparse.
\begin{theorem}
\label{thm:sparseval}
Let $\bS$ denote a randomly chosen holdout set of size $n$.
Let $\A$ be an algorithm that is given access to $\Sparse(m,B)$ and outputs queries $\psi_1,\ldots,\psi_m$ such that each $\psi_i$ is in some set $\Psi_i$ of functions from $\X^n$ to $\zo$. Assume that for every $i\in[m]$ and  $\psi_i \in \Psi_i$, $\pr[\psi_i(\bS) = 1] \leq \beta_i$.
Let $\bm{\psi}_i$ be the random variable equal to the $i$'th query of $\A$ on $\bS$. Then
$ \pr[ \bm{\psi}_i(\bS) = 1] \leq \ell_i \cdot \beta_i, $ where $\ell_i = \sum_{j=0}^{\min\{i-1,B\}} {i \choose j} \leq m^B$.
\end{theorem}
\iffull
\begin{proof}
Let $\B$ denote the algorithm that represents view the interaction of $\A$ with $\Sparse(m,B)$ up until query $i$ and outputs the all the $i-1$ responses of $\Sparse(m,B)$ in this interaction. If there are $B$ responses with value 1 in the interaction then all the responses after the last one are meaningless and can be assumed to be equal to 0. The number of binary strings of length $i-1$ that contain at most $B$ ones is exactly $\ell_i = \sum_{j=0}^{\min\{i-1,B\}} {i \choose j}$. Therefore we can assume that the output domain of $\B$ has size $\ell_i$ and we denote it by $\Y$. Now, for $y \in \Y$ let $R(y)$ be the set of datasets $S$ such that $\psi_i(S) =1$, where $\psi_i$ is the function that $\A$ generates when the responses of $\Sparse(m,B)$ are $y$ and the input holdout dataset is $S$ (for now assume that $\A$ is deterministic).
By the conditions of the theorem we have that for every $y$, $\pr[\bS \in R(y)] \leq \beta_i$. Applying Thm.~\ref{thm:descr-basic} to $\B$, we get that $\pr[\bS \in R(\B(\bS))] \leq \ell_i \beta_i$, which is exactly the claim.  We note that to address the case when $\A$ is randomized (including dependent on the random choice of the training set) we can use the argument above for every fixing of all the random bits of $\A$. From there we obtain that the claim holds when the probability is taken also over the randomness of $\A$.

We remark that the proof can also be obtained via a more direct application of the union bound over all strings in $\Y$. But the proof via Thm.~\ref{thm:descr-basic} demonstrates the conceptual role that short description length plays in this application.
\end{proof}
\fi
In this general formulation it is the analyst's responsibility to use the budgets economically and pick query functions that do not fail validation often. At the same time, \Sparse ensures that (for the appropriate values of the parameters) the analyst can think of the holdout set as a fresh sample for the purposes of validation. Hence the analyst can pick queries in such a way that failing the validation reliably indicates overfitting. To relate this algorithm to \Tho, consider the validation query function that is the indicator of the condition $|\cE_{S_h}[\phi]-\cE_{S_t}[\phi]|>\hat{T}+\eta$ (note that this condition can be evaluated using an algorithm with access to $S_h$). This is precisely the condition that consumes the overfitting budget of $\Tho$. Now, as in \Tho, for every fixed $\phi$, $\pr[|\cE_{S_h}[\phi] - \calP[\phi]|\geq\tau] \leq 2e^{-2\tau^2n}$. If $B \leq \tau^2 n/\ln m$, then we obtain that for every query $\phi$ generated by the analyst, we still have strong concentration of the mean on the holdout set around the expectation: $\pr[|\cE_{S_h}[\phi] - \calP[\phi]|\geq\tau] \leq 2e^{-\tau^2n}$. This implies that if the condition $|\cE_{S_h}[\phi]-\cE_{S_t}[\phi]|>\hat{T}+\eta$ holds, then with high probability also the condition $|\calP[\phi]-\cE_{S_t}[\phi]|>\hat{T}+\eta-\tau$ holds, indicating overfitting.
 \iffull
One notable distinction of $\Tho$ from $\Sparse$ is that $\Sparse$ does not provide corrections in the case of overfitting. One way to remedy that is simply to use a version of \Sparse that allows functions with values in $\{0,1,\ldots,L\}$. It is easy to see that for such functions we would obtain the bound of the form $\ell_i \cdot L^{\min\{B,i-1\}}
\cdot \beta_i$. To output a value in $[0,1]$ with precision $\tau$, $L = \lfloor 1/\tau \rfloor$ would suffice. However, in many cases a more economical solution would be to have a separate dataset which is used just for obtaining the correct estimates.

 \fi
 An example of the application of \Sparse for answering statistical and low-sensitivity queries that is based on our analysis can be found in \cite{BassilySSU15}. The analysis of generalization on the holdout set in \cite{BlumH15} and the analysis of the Median Mechanism we give in Section \ref{sec:median} also rely on this sparsity-based technique.
\iffull

An alternative view of this algorithm is as a general template for designing algorithms for answering some specific type of adaptively chosen queries. Generalization guarantees specific to the type of query can then be obtained from our general analysis. For example, an algorithm that fits a mixture of Gaussians model to the data could define the validation query to be an algorithm that fits the mixture model to the holdout and obtains a vector of parameters $\Theta_h$. The validation query then compares it with the vector of parameters $\Theta_t$ obtained on the training set and outputs 1 if the parameter vectors are ``not close" (indicating overfitting). Given guarantees of statistical validity of the parameter estimation method in the static setting one could then derive guarantees for adaptive validation via Thm.~\ref{thm:sparseval}.
\fi
\iffull
\else
\paragraph{Experiments:}
\label{sec:experiments}
In our experiment the analyst is given a $d$-dimensional labeled data set $S$ of size $2n$ and splits it
randomly into a training set $S_t$ and a holdout set $S_h$ of equal size. We
denote an element of $S$ by a tuple $(x,y)$ where $x$ is a $d$-dimensional
vector and $y\in\{-1,1\}$ is the corresponding class label.
The analyst wishes to select variables to be included in her classifier. For various values of the number of variables to select $k$,
she picks $k$ variables with the largest absolute correlations with the label. However, she verifies the correlations (with the label) on the holdout set and uses only those variables whose correlation agrees in sign with the correlation on the training set and both correlations are larger than some threshold in absolute value. She then creates a simple linear threshold classifier on the selected variables using only the signs of the correlations of the selected variables (see supplemental material for a formal description). A final test evaluates the classification accuracy of the classifier on both the training set and the holdout set.

In our first experiment, each attribute of $x$ is drawn independently from the normal distribution $N(0,1)$ and we choose the class label $y \in \{-1,1\}$ uniformly at random so that there is no correlation between the data
point and its label. We chose $n=10,000$, $d=10,000$ and varied the number of selected variables $k$.
In this scenario no classifier can achieve
true accuracy better than 50\%. Nevertheless, reusing a standard holdout results
in reported accuracy of over 63\% for $k=500$ on both the training set and the
holdout set (the standard deviation of the error is less than 0.5\%). The average and standard deviation of results obtained from $100$ independent executions of the experiment are plotted below. For comparison, the plot also includes the accuracy of the classifier on another fresh data set of size $n$ drawn from the same distribution. We then executed the same algorithm with our reusable holdout. \Tho was invoked with $T=0.04$ and $\tau = 0.01$ explaining why the accuracy of the classifier reported by \Tho is off by up to $0.04$ whenever the accuracy on the holdout set is within $0.04$ of the accuracy on the training set.  We also used Gaussian noise instead of Laplacian noise as it has stronger
concentration properties. \Tho prevents the algorithm from overfitting to the holdout set and gives a valid estimate of classifier accuracy.
Additional experiments and discussion are presented in the supplemental material.

\begin{figure}[h]
\begin{center}
\includegraphics[width=\textwidth]{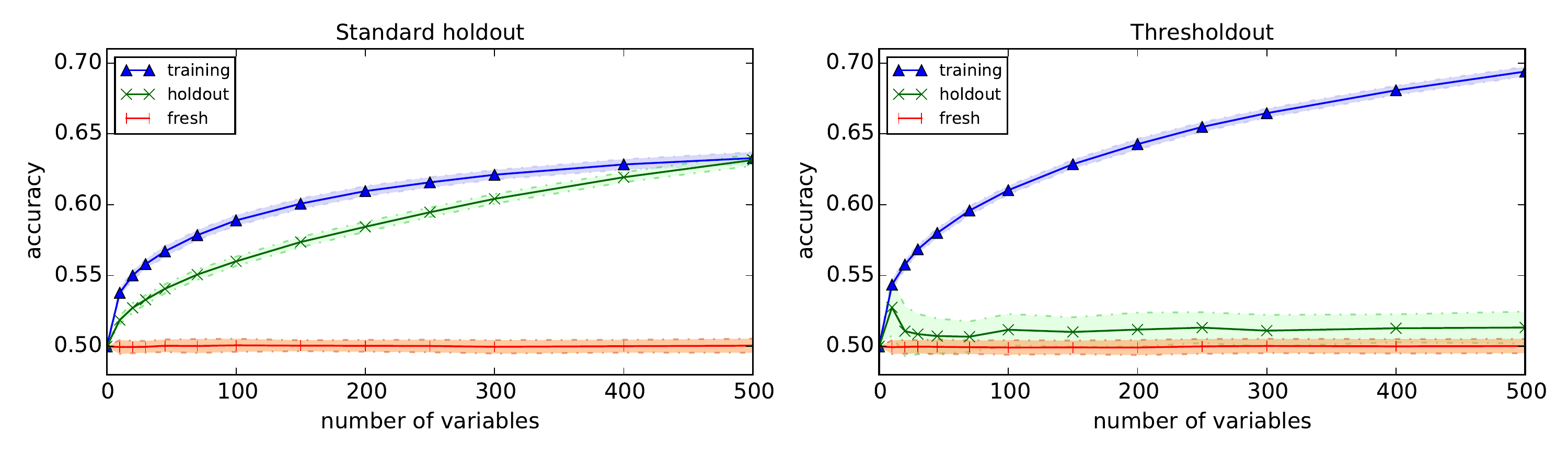}
\end{center}
\remove{
\caption{No correlation between class labels and data points.
The plot shows
the classification accuracy of the classifier on training,
holdout and fresh sets. Margins indicate the standard deviation.}}
\label{fig:exp1} 
\end{figure}
\fi

%% file: experiment-info-reuse.tex
\section{Experiments}
\label{sec:experiments}  We describe a simple experiment on synthetic data that illustrates
 the danger of reusing a standard holdout set and how this issue can be
resolved by our reusable holdout.
In our experiment the analyst wants to build a classifier via the following common strategy. First the analyst finds a set of single attributes that are correlated with the class label. Then the
analyst aggregates the correlated variables into a single model of higher
accuracy (for example using boosting or bagging methods).
More formally, the analyst is given a $d$-dimensional labeled data set $S$ of size $2n$ and splits it
randomly into a training set $S_t$ and a holdout set $S_h$ of equal size. We
denote an element of $S$ by a tuple $(x,y)$ where $x$ is a $d$-dimensional
vector and $y\in\{-1,1\}$ is the corresponding class label.
The analyst wishes to select variables to be included in her classifier. For various values of the number of variables to select $k$,
she picks $k$ variables with the largest absolute correlations with the label. However, she verifies the correlations (with the label) on the holdout set and uses only those variables whose correlation agrees in sign with the correlation on the training set and both correlations are larger than some threshold in absolute value. She then creates a simple linear threshold classifier on the selected variables using only the signs of the correlations of the selected variables. A final test evaluates the classification accuracy of the classifier on both the training set and the holdout set.

Formally, the algorithm is used to build a linear threshold classifier:
\begin{enumerate}
\item For each attribute $i\in[d]$ compute the correlation with the label on the training and holdout sets: $w_i^t =
\sum_{(x,y)\in S_t}x_i y$ and $w_i^h = \sum_{(x,y)\in S_h}x_i y$.
Let \[W= \left\{i\ |\  w_i^t \cdot w_i^h > 0;\ |w_i^t|\geq 1/\sqrt{n}; \ |w_i^h|\geq 1/\sqrt{n}| \right\}\] that is the set of variables for which $w_i^t$ and $w_i^h$ have the same sign and both are at least $1/\sqrt{n}$ in absolute value (this is the standard deviation of the correlation in our setting). Let $V_k$ be the subset of variables in $W$ with $k$ largest values of $|w_i^t|$.
\item Construct the classifier
$f(x) = \sgn\left( \sum_{i \in V_k} \sgn(w_i^t) \cdot x_i \right)$.
\end{enumerate}

In the experiments we used an implementation of \Tho that differs somewhat from the algorithm
we analyzed theoretically (given in Figure \ref{fig:Tho}). Specifically, we set the parameters to be
 $T=0.04$ and $\tau = 0.01$. This is lower than the values necessary for the proof (and which are not intended for direct application) but suffices to prevent overfitting in our experiment. Second, we use Gaussian noise instead of Laplacian noise as it has stronger
concentration properties (in many differential privacy applications similar theoretical guarantees hold mechanisms based on Gaussian noise).

\noindent{\bf No correlation between labels and data:}
In our first experiment, each attribute is drawn independently from the normal distribution $N(0,1)$ and we choose the class label $y \in \{-1,1\}$ uniformly at random so that there is no correlation between the data
point and its label. We chose $n=10,000$, $d=10,000$ and varied the number of selected variables $k$.
In this scenario no classifier can achieve
true accuracy better than 50\%. Nevertheless, reusing a standard holdout results
in reported accuracy of over 63\% for $k=500$ on both the training set and the
holdout set (the standard deviation of the error is less than 0.5\%). The average and standard deviation of results obtained from $100$ independent executions of the experiment are plotted in Figure \ref{fig:exp1} which also includes the accuracy of the classifier on another fresh data set of size $n$ drawn from the same distribution. We then executed the same algorithm with our reusable holdout. The algorithm \Tho was invoked with $T=0.04$ and $\tau = 0.01$ explaining why the accuracy of the classifier reported by \Tho is off by up to $0.04$ whenever the accuracy on the holdout set is within $0.04$ of the accuracy on the training set. \Tho prevents the algorithm from overfitting to the holdout set and gives a valid estimate of classifier accuracy.

\begin{figure}[h]
\begin{center}
\includegraphics[width=\textwidth]{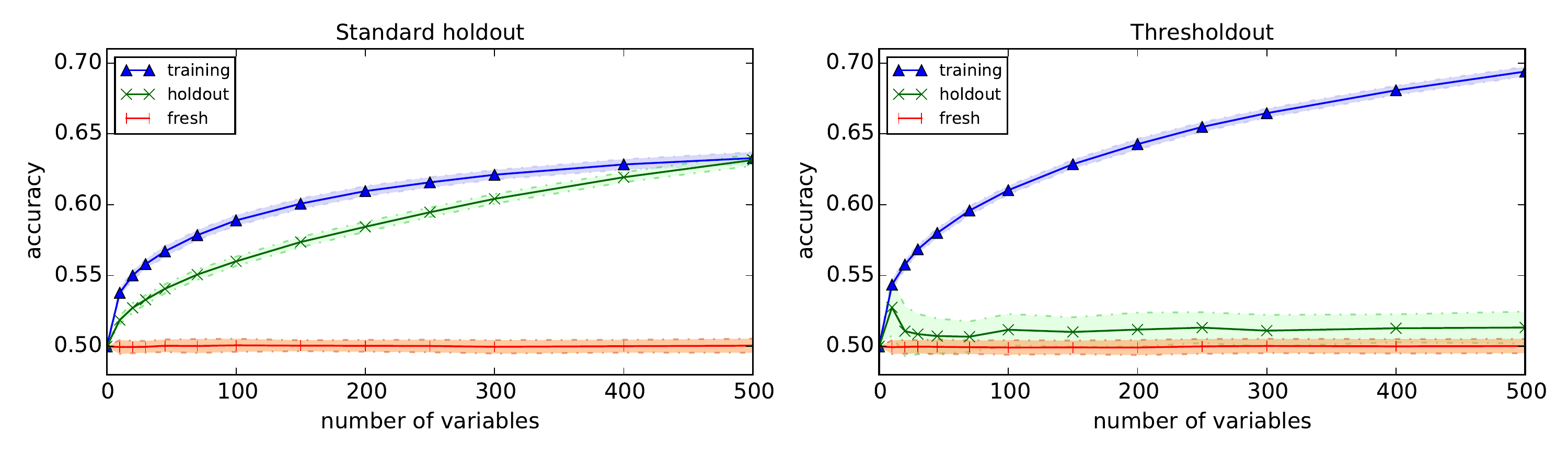}
\end{center}
\caption{No correlation between class labels and data points. The plot shows
the classification accuracy of the classifier on training,
holdout and fresh sets. Margins indicate the standard deviation.}
\label{fig:exp1} 
\end{figure}
\noindent{\bf High correlation between labels and some of the variables:}
In our second experiment, the class labels
are correlated with some of the variables. As before the label is randomly chosen from $\{-1,1\}$ and each of the attributes is drawn from $N(0,1)$ aside from $20$ attributes which are drawn from $N(y \cdot 0.06,1)$ where $y$ is the class label. We execute the same algorithm on this data with both the standard holdout and \Tho and plot the results in Figure \ref{fig:exp2}.
Our experiment shows that when using the
reusable holdout, the algorithm still finds a good classifier while
preventing overfitting. This illustrates that the reusable holdout
simultaneously prevents overfitting and allows for the discovery of true
statistical patterns.

\begin{figure}[h]
\begin{center}
\includegraphics[width=\textwidth]{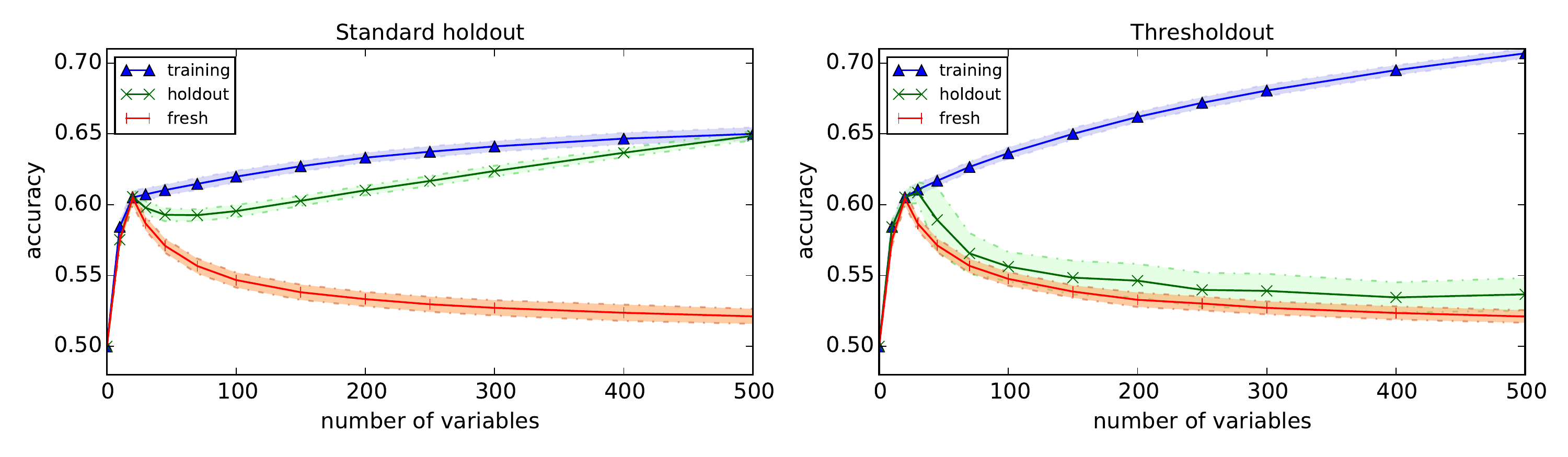}
\end{center}
\caption{Some variables are correlated with the label.}
\label{fig:exp2} 
\end{figure}

In Figures \ref{fig:exp1} and \ref{fig:exp2}, simulations that used \Tho for selecting the variables also show the accuracy on the holdout set as reported by \Tho. For comparison purposes, in
Figure \ref{fig:exp3} we plot the actual accuracy of the generated classifier on the holdout set (the parameters of the simulation are identical to those used in Figures \ref{fig:exp1} and \ref{fig:exp2}). It demonstrates that there is essentially no overfitting to the holdout set. Note that the advantage of the accuracy reported by \Tho is that it can be used to make further data dependent decisions while mitigating the risk of overfitting.

\begin{figure}[h]
\begin{center}
\includegraphics[width=\textwidth]{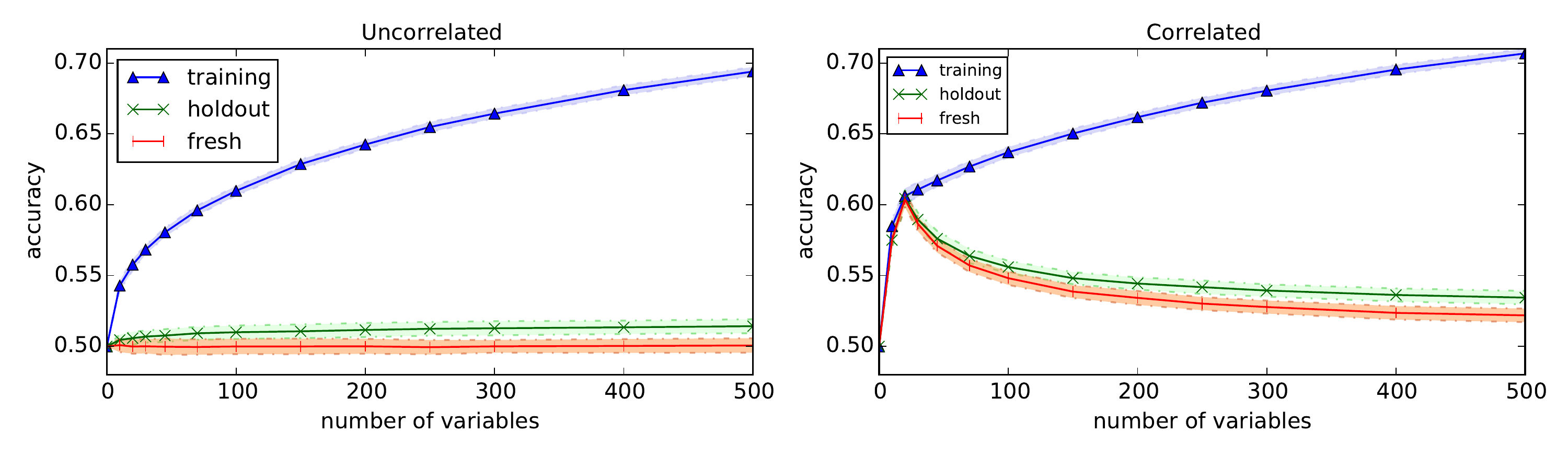}
\end{center}
\caption{Accuracy of the classifier produced with \Tho on the holdout set.}
\label{fig:exp3} 
\end{figure}

\noindent{\bf Discussion of the results:} Overfitting to the standard holdout set
arises in our experiment because the analyst reuses the holdout after using
it to measure the correlation of single attributes. We first note that
neither cross-validation nor bootstrap resolve this issue. If we used either
of these methods to validate the correlations, overfitting would still arise
due to using the same data for training and validation (of the final
classifier). It is tempting to recommend other solutions to the specific
problem we used in our experiment. Indeed, a significant number of methods in
the statistics and machine learning literature deal with inference for fixed two-step procedures
where the first step is variable selection (see \cite{hastie2009elements} for
examples). Our experiment demonstrates that even in such simple and standard
settings our method avoids overfitting without the need to use a
specialized procedure -- and, of course, extends more broadly. More
importantly, the reusable holdout gives the analyst a general and principled
method to perform multiple validation steps where previously the only known
safe approach was to collect a fresh holdout set each time a function
depends on the outcomes of previous validations.

%% file: maxinto2dl-info-reuse.tex
\section{From Max-information to Randomized Description Length}
\label{sec:max-info2dl}
In this section we demonstrate additional connections between max-information, differential privacy and description length. These connections are based on a generalization of description length to randomized algorithms that we refer to as {\em randomized description length.}
\begin{definition}
For a universe $\Y$ let $\A$ be a randomized algorithm with input in $\X$ and output in $\Y$.
We say that the output of $\A$ has randomized description length $k$ if for every fixed setting of random coin flips of $\A$ the set of possible outputs of $\A$ has size at most $2^k$.
\end{definition}
We first note that just as the (deterministic) description length, randomized description length implies generalization and gives a bound on max-information.
\begin{theorem}
\label{thm:descr-basic-random}
Let $\A:\X^n \rightarrow \Y$ be an algorithm with randomized description length $k$ and let $\bS$ be a random dataset over $\X^n$. Assume that $R:\Y \rightarrow 2^{\X^n}$ is such that for every $y \in \Y$, $\pr[\bS \in R(y)] \leq \beta$. Then $\pr[\bS \in R(\A(\bS))] \leq 2^k \cdot \beta$.
\end{theorem}

\begin{theorem}
\label{thm:randdl2maxinfo}
Let $\A$ be an algorithm with randomized description length $k$ taking as an input an $n$-element dataset and outputting a value in $\Y$.  Then for every $\beta > 0$, $I_{\infty}^\beta(\A,n)\leq \log(|\Y|/\beta)$.
\end{theorem}
\begin{proof}
Let $\bS$ be any random variable over $n$-element input datasets and let $\bY$ be the corresponding output distribution
$\bY = \A(\bS)$. It suffices to prove that for every $\beta > 0$, $I_{\infty}^\beta(\bS;\bY)\leq k+\log(1/\beta)$.

Let $R$ be the set of all possible values of the random bits of $\A$ and let $\cR$ denote the uniform distribution over a choice of $r \in R$. For $r\in R$, let $\A_r$ denote $\A$ with the random bits set to $r$ and let $\bY_r = \A_r(\bS)$.
Observe that by the definition of randomized description length, the range of $\A_r$ has size at most $2^k$. Therefore, by Theorem \ref{thm:dl2maxinfo}, we obtain that $I_\infty^\beta(\bS;\bY_r) \leq  \log (2^k/\beta)$.

For any event $\cO \subseteq \X^n \times \Y$ we have that
\alequn{\pr[(\bS,\bY) \in \cO] &=\E_{r\sim \cR}[\pr[(\bS,\bY_r) \in \cO]] \\
& \leq \E_{r\sim \cR}\left[\frac{2^k}{\beta} \cdot \pr[\bS \times \bY_r \in \cO] + \beta\right] \\
& = \frac{2^k}{\beta} \cdot \E_{r\sim \cR}\left[ \pr[\bS \times \bY_r \in \cO] \right] + \beta\\
& =  \frac{2^k}{\beta} \cdot \pr[\bS\times \bY \in \cO] + \beta .
}
By the definition of $\beta$-approximate max-information, we obtain that
$I_\infty^\beta(\bS;\bY) \leq  \log (2^k/\beta)$.
\end{proof}

We next show that if the output of an algorithm $\A$ has low approximate max-information about its input then there exists a (different) algorithm whose output is statistically close to that of $\A$ while having short randomized description. We remark that this reduction requires the knowledge of the marginal distribution $\A(\bS)$.

\begin{lemma}
\label{lem:subsample}
Let $\A$ be a randomized algorithm taking as an input a dataset of $n$ points from $\X$ and outputting a value in $\Y$.
Let $\bZ$ be a random variable over $\Y$. For $k>0$ and a dataset $S$ let $\beta_S = \min\{\beta \cond D_\infty^\beta(\A(S)\|\bZ) \leq k\}$. There exists an algorithm $\A'$ that given $S \in \X^n$, $\beta,k$ and any $\beta'>0$,
\begin{enumerate}
\item the output of $\A'$ has randomized description length $k+\log\ln(1/\beta')$.
\item for every $S$, $\Delta(\A'(S), \A(S)) \leq \beta_S + \beta'$.
\end{enumerate}
\end{lemma}
\begin{proof}
Let $S$ denote the input dataset. By definition of $\beta_S$, $D_\infty^{\beta_S}(\A(S)\|\bZ) \leq k$.
By the properties of approximate divergence (\eg \cite{DworkR14}), $D_\infty^{\beta_S}(\A(S)\|\bZ) \leq k$ implies that there exists a
random variable $\bY$ such that $\Delta(\A(S),\bY) \leq \beta_S$ and $D_\infty(\bY\|\bZ) \leq k$.

For $t= 2^k \ln(1/\beta')$ the algorithm $\A'$ randomly and independently chooses $t$ samples from $\bZ$.  Denote them by $y_1,y_2,\ldots,y_t$. For $i=1,2,\ldots,t$, $\A'$ outputs $y_i$ with probability $p_i = \frac{\pr[\bY=y_i]}{2^k \cdot \pr[\bZ=y_i]}$ and goes to the next sample otherwise. Note that $p_i \in [0,1]$ and therefore this is a legal choice of probability. When all samples are exhausted the algorithm outputs $y_1$.

We first note that by the definition of this algorithm its output has randomized description length $\log t = k+\log\ln(1/\beta')$.
Let $T$ denote the event that at least one of the samples was accepted. Conditioned on this event the output of $\A'(S)$ is distributed
according to $\bY$.
For each $i$, $$\E_{y_i \sim p(\bZ)}[p_i]=\E_{y_i \sim p(\bZ)}\left[\frac{\pr[\bY=y_i]}{2^k \cdot \pr[\bZ=y_i]}\right] = \sum_{y_i \in \Y }\frac{\pr[\bY=y_i]}{2^k} = \frac{1}{2^k}.$$
This means that the probability that none of $t$ samples will be accepted is $(1-2^{-k})^t \leq e^{t/2^k} \leq \beta'$. Therefore
$\Delta(\A'(S),\bY) \leq \beta'$ and, consequently, $\Delta(\A'(S),\A(S)) \leq \beta_S + \beta'$.
\end{proof}

We can now use Lemma \ref{lem:subsample} to show that if for a certain random choice of a dataset $\bS$, the output of $\A$ has low approximate max-information then there exists an algorithm $\A'$ whose output on $\bS$ has low randomized description length and is statistically close to the output distribution of $\A$.

\begin{theorem}
\label{thm:info2dl}
Let $\bS$ be a random dataset in $\X^n$ and $\A$ be an algorithm taking as an input a dataset in $\X^n$ and having a range $\Y$. Assume that for some $\beta \geq 0$, $I_{\infty}^\beta(\bS;\A(\bS))=k$. For any $\beta'>0$, there exists an algorithm $\A'$ taking as an input a dataset in $\X^n$  such that
\begin{enumerate}
\item the output of $\A'$ has randomized description length $k+\log\ln(1/\beta')$;
\item $\Delta((\bS,\A'(\bS)), (\bS,\A(\bS)) \leq \beta + \beta'$.
\end{enumerate}
\end{theorem}
\begin{proof}
For a dataset $S$ let $\beta_S = \min\{\beta \cond D_\infty^\beta(\A(S)\|\A(\bS)) \leq k\}$.
To prove this result it suffices to observe that $\E[\beta_\bS] \leq \beta$  and then apply Lemma \ref{lem:subsample} with $\bZ = \A(\bS)$
To show that $\E[\beta_\bS] \leq \beta$ let $\cO_S \subseteq \Y$ denote an event such that $\pr[\A(S) \in \cO_S] = 2^k \cdot \pr[\A(\bS) \in \cO_S] + \beta_S$. Let $\cO = \bigcup_{S\in \X^n} \{(S,\cO_S)\}$. Then,
\alequn{\pr[(\bS,\A(\bS) \in \cO] &= \E_{S \sim p(\bS)}\left[\pr[(S,\A(S) \in \cO_S]\right] \\
&= \E_{S \sim p(\bS)}\left[2^k \cdot \pr[\A(\bS) \in \cO_S] + \beta_S\right] \\
&  = 2^k \cdot \pr[\bS \times \A(\bS) \in \cO] + \E[\beta_\bS] .  }
If $\E[\beta_\bS] > \beta$ then it would hold for some $k' > k$ that $\pr[(\bS,\A(\bS) \in \cO] = 2^{k'} \cdot \pr[\bS \times \A(\bS) \in \cO] + \beta$ contradicting the assumption $I_{\infty}^\beta(\bS;\A(\bS))=k$. We remark that, it is also easy to see that $\E[\beta_\bS] =\beta$.
\end{proof}

It is important to note that Theorem \ref{thm:info2dl} is not the converse of Theorem \ref{thm:randdl2maxinfo} and does not imply equivalence between max-information and randomized description length. The primary difference is that Theorem \ref{thm:info2dl} defines a new algorithm rather than arguing about the original algorithm. In addition, the new algorithm requires samples of $\A(\bS)$, that is, it needs to know the marginal distribution on $\Y$. As a more concrete example, Theorem \ref{thm:info2dl} does not allow us to obtain a description-length-based equivalent of Theorem \ref{thm:concentrated-divergence} for all i.i.d.~datasets.
On the other hand, any algorithm that has bounded  max-information for all distributions over datasets can be converted to an algorithm with low randomized description length.
\begin{theorem}
\label{thm:dp2dl}
Let $\A$ be an algorithm over $\X^n$ with range $\Y$ and let $k= I_\infty(\A,n)$.  For any $\beta>0$, there exists an algorithm $\A'$ taking as an input a dataset in $\X^n$  such that
\begin{enumerate}
\item the output of $\A'$ has randomized description length $k + \log\ln(1/\beta)$;
\item for every dataset $S\in \X^n$, $\Delta(\A'(S), \A(S)) \leq \beta$.
\end{enumerate}
\end{theorem}
\begin{proof}
Let $S_0 = (x,x,\ldots,x)$ be an $n$-element dataset for an arbitrary $x \in \X$.
By Lemma \ref{lem:max-info-pairwise} we know that for every $S\in \X^n$,
$D_\infty(\A(S) \| \A(S_0))\leq k$. We can now apply Lemma \ref{lem:subsample} with $\bZ = \A(S_0)$, and $\beta' = \beta$ to obtain the result.
\end{proof}

The conditions of Theorem \ref{thm:dp2dl} are satisfied by any $\eps$-differentially private algorithm with $k= \log e\cdot  \eps n$. This immediately implies that the output of any $\eps$-differentially private algorithm is $\beta$-statistically close to the output of an algorithm with randomized description length of $\log e\cdot  \eps n + \log\ln(1/\beta)$. Special cases of this property have been derived (using a technique similar to Lemma \ref{lem:subsample}) in the context of proving lower bounds for learning algorithms \cite{BeimelNS:13} and communication complexity of differentially private protocols \cite{McGregorMPRTV10}.

%% file: desclengthquerybounds.tex
\section{Answering Queries via Description Length Bounds}
\label{sec:median}
In this section, we show a simple method for answering any adaptively chosen sequence of $m$  statistical queries, using a number of samples that scales only polylogarithmically in $m$. This is an exponential improvement over what would be possible by naively evaluating the queries exactly on the given samples. Algorithms that achieve such dependence were given in \cite{DworkFHPRR14:arxiv} and \cite{BassilySSU15} using differentially private algorithms for answering queries and the connection between generalization and differential privacy (in the same way as we do in Section \ref{sec:holdout-tho}). Here we give a simpler algorithm which we analyze using description length bounds. The resulting bounds are comparable to those achieved in \cite{DworkFHPRR14:arxiv} using pure differential privacy but are somewhat weaker than those achieved using approximate differential privacy \cite{DworkFHPRR14:arxiv,BassilySSU15,NissimS15}.

The mechanism we give here is based on the \emph{Median Mechanism} of Roth and Roughgarden \cite{RR10}. A differentially private variant of this mechanism was introduced in \cite{RR10} to show that it was possible to answer exponentially many adaptively chosen {\em counting} queries (these are queries for an estimate of the empirical mean of a function $\phi:\X\rightarrow [0,1]$ on the dataset).  Here we analyze a noise-free version and establish its properties via a simple description length-based argument. We remark that it is possible to analogously define and analyze the noise-free version of the Private Multiplicative Weights Mechanism of Hardt and Rothblum \cite{HardtR10}. This somewhat more involved approach would lead to better (but qualitatively similar) bounds.

Recall that statistical queries are defined by functions $\phi:\X\rightarrow [0,1]$, and our goal is to correctly estimate their expectation $\calP[\phi]$. The Median Mechanism takes as input a dataset $S$ and an adaptively chosen sequence of such functions $\phi_1,\ldots,\phi_m$, and outputs a sequence of answers $a_1,\ldots,a_m$.
\begin{figure}[h]
\begin{boxedminipage}{\textwidth}
\textbf{Algorithm} Median Mechanism

\noindent\textbf{Input:} An upper bound $m$ on the total number of queries, a dataset $S$ and an accuracy parameter $\tau$

\begin{enumerate}
    \item \textbf{Let} $\alpha = \frac{\tau}{3}$.
    \item \textbf{Let} $\textrm{Consistent}_0 = \X^{\frac{\log m}{\alpha^2}}$
    \item \textbf{For} a query $\phi_i$ \textbf{do}
    \begin{enumerate}
        \item \textbf{Compute} $a^{\textrm{pub}}_i = \mathrm{median}(\{\cE_{S'}[\phi_i] : S' \in \textrm{Consistent}_{i-1}\})$.
        \item \textbf{Compute} $a^{\textrm{priv}}_i = \cE_{S}[\phi_i]$.
        \item \textbf{If} $\left|a^{\textrm{pub}}_i - a^{\textrm{priv}}_i\right| \leq 2\alpha$ \textbf{Then}:
        \begin{enumerate}
          \item \textbf{Output} $a_i = a^{\textrm{pub}}_i$.
          \item \textbf{Let} $\textrm{Consistent}_{i} = \textrm{Consistent}_{i-1}$.
        \end{enumerate}
        \item \textbf{Else}:
        \begin{enumerate}
          \item \textbf{Output} $a_i = \lfloor a^{\textrm{priv}}_i \rfloor_{\alpha}$.
          \item \textbf{Let} $\textrm{Consistent}_{i} = \{S' \in \textrm{Consistent}_{i-1} : |a_i - \cE_{S'}[\phi_i]| \leq 2\alpha\}$.
        \end{enumerate}
   \end{enumerate}
\end{enumerate}
\end{boxedminipage}
\caption{Noise-free version of the Median Mechanism from \cite{RR10} 
}
\label{fig:MMech}
\end{figure}

The guarantee we get for the Median Mechanism is as follows:
\begin{theorem}
Let $\beta,\tau>0$ and $m \geq B > 0$. Let $\bS$ denote a dataset of size $n$ drawn i.i.d.~from a distribution~$\calP$ over $\X$. Consider any algorithm that adaptively chooses functions $\bphi_1,\dots,\bphi_m$ while interacting with the Median Mechanism which is given
$\bS$ and values $\tau$, $\beta$. For every $i\in [m]$, let $\ba_i$ denote the answer of the Median Mechanism on function~$\bphi_i:\X\rightarrow [0,1]$.
Then \[
\Pr\left[\exists i \in [m],\ \left|\ba_i-\calP[\bphi_i]\right|\geq \tau\right] \le \beta\
\]
whenever
$$n \geq n_0 = \frac{81 \cdot \log |\X|\cdot \log m \cdot \ln (3m/\tau)}{2\tau^4}  + \frac{9 \ln(2m/\beta)}{2\tau^2} .$$
\end{theorem}
\begin{proof}
We begin with a simple lemma which informally states that for every distribution $\calP$, and for every set of $m$ functions $\phi_1,\ldots,\phi_m$, there exists a small dataset that approximately encodes the answers to each of the corresponding statistical queries.
\begin{lemma}[\cite{DworkR14} Theorem 4.2]
\label{lem:smallDB}
For every dataset $S$ over $\X$, any set $C = \{\phi_1,\ldots,\phi_m\}$ of $m$ functions $\phi_i:\X\rightarrow [0,1]$, and any $\alpha \in [0,1]$, there exists a data set $S' \in \X^t$ of size $t = \frac{\log m}{\alpha^2}$ such that:
$$\max_{\phi_i \in C}\left|\cE_{S}[\phi_i] - \cE_{S'}[\phi_i]\right| \leq \alpha .$$
\end{lemma}

Next, we observe that by construction, the Median Mechanism (as presented in Figure \ref{fig:MMech}) always returns answers that are close to the empirical means of respective query functions.
\begin{lemma}
\label{lem:empirical}
For every sequence of queries $\phi_1,\ldots,\phi_m$ and dataset $S$ given to the Median Mechanism, we have that for every $i$:
$$\left|a_i - \cE_{S}[\phi_i]\right| \leq 2\alpha .$$
\end{lemma}

Finally, we give a simple lemma from \cite{RR10} that shows that the Median Mechanism only returns answers computed using the dataset $S$ in a small number of rounds -- for any other round $i$, the answer returned is computed from the set $\textrm{Consistent}_{i-1}$.
\begin{lemma}[\cite{RR10}, see also Chapter 5.2.1 of \cite{DworkR14}]
\label{lem:hardqueries}
For every sequence of queries $\phi_1,\ldots,\phi_m$ and a dataset $S$ given to the Median Mechanism:
$$\left|\{i : a_i \neq a^{\textrm{pub}}_i \}\right| \leq \frac{\log|\X|\log m}{\alpha^2} .$$
\end{lemma}
\begin{proof}
We simply note several facts. First, by construction, $|\textrm{Consistent}_0| = \left|\X\right|^{\log m/\alpha^2}$. Second, by Lemma \ref{lem:smallDB}, for every $i$, $|\textrm{Consistent}_i| \geq 1$ (because for every set of $m$ queries, there is at least one dataset $S'$ of size $\log m/\alpha^2$ that is consistent up to error $\alpha$ with $S$ on every query asked, and hence is never removed from $\textrm{Consistent}_i$ on any round). Finally, by construction, on any round $i$ such that $a_i \neq a^{\textrm{pub}}_i$, we have $|\textrm{Consistent}_i| \leq \frac{1}{2}\cdot |\textrm{Consistent}_{i-1}|$ (because on any such round, the median dataset $S'$ -- and hence at least half of the datasets in $\textrm{Consistent}_{i-1}$ were inconsistent with the answer given, and hence removed.) The lemma follows from the fact that there can therefore be at most $\log\left( \left|\X\right|^{\log m/\alpha^2}\right)$ many such rounds.
\end{proof}

Our analysis proceeds by viewing the interaction between the data analyst and the Median Mechanism (Fig.~\ref{fig:MMech}) as a single algorithm $\A$.
 $\A$ takes as input a dataset $S$ and outputs a set of queries and answers $\A(S) = \{\phi_i,a_i\}_{i=1}^m$. We will show that $\A$'s output has short randomized description length (the data analyst is a possibly randomized algorithm and hence $\A$ might be randomized).

\begin{lemma}
\label{lem:mmDL}
Algorithm $\A$ has randomized description length of at most
$$b \leq \frac{\log|\X|\cdot\log m}{\alpha^2}\cdot\left(\log m + \log \frac{1}{\alpha}\right)$$ bits.
\end{lemma}
\begin{proof}
We observe that for every fixing of the random bits of the data analyst the entire sequence of queries asked by the analyst, together with the answers he receives, can be reconstructed from a record of the \emph{indices} $i$ of the queries $\phi_i$ such that $a_i \neq a^{\textrm{pub}}_i$, together with their answers $a_i$ ({\em i.e.}~it is sufficient to encode $M := \{(i, a_i) \cond  a_i \neq a^{\textrm{pub}}_i\}$. Once this is established, the lemma follows because by Lemma \ref{lem:hardqueries}, there are at most $\left(\frac{\log|\X|\log m}{\alpha^2}\right)$ such queries, and for each one, its index can be encoded with $\log m$ bits, and its answer with $\log \frac{1}{\alpha}$ bits.

To see why this is so, consider the following procedure for reconstructing the sequence $(\phi_1,a_1,\ldots,\phi_m,a_m)$ of queries asked and answers received. 
For every fixing of the random bits of the data analyst, her queries can be expressed as a sequence of functions $(f_1,\ldots,f_m)$ that take as input the queries previously asked to the Median Mechanism, and the answers previously received, and output the next query to be asked. That is, we have:
$$f_1() := \phi_1,\ \  f_2(\phi_1,a_1) := \phi_2,\ \  f_3(\phi_1,a_1,\phi_2,a_2) := \phi_3,\ \ \ldots,\ \ f_m(\phi_1,a_1,\ldots,\phi_{m-1},a_{m-1}) := \phi_m .$$ Assume inductively that at stage $i$, the procedure has successfully reconstructed $(\phi_1,a_1,\ldots,\phi_{i-1},a_{i-1},\phi_i)$, and the set $\textrm{Consistent}_{i-1}$ (This is trivially satisfied at stage $i = 1$). For the inductive case, we need to recover $a_i, \phi_{i+1}$, and $\textrm{Consistent}_i$. There are two cases we must consider at stage $i$. In the first case, $i$ is such that $a_i \neq a^{\textrm{pub}}_i$. But in this case, $(i, a_i) \in M$ by definition, and so we have recovered $a_i$, and we can compute $\phi_{i+1} = f_{i+1}(\phi_1,a_1,\ldots,\phi_1,a_i)$, and can compute $\textrm{Consistent}_{i} = \{S' \in \textrm{Consistent}_{i-1} : |a_i - \cE_{S'}[\bphi_i]| \leq 2\alpha\}$. In the other case, $a_i = a^{\textrm{pub}}_i$. But in this case, by definition of $a^{\textrm{pub}}_i$, we can compute $a_i = \mathrm{median}(\{\cE_{S'}[\phi_i] : S' \in \textrm{Consistent}_{i-1}\})$, $\phi_{i+1} = f_{i+1}(\phi_1,a_1,\ldots,\phi_i,a_i)$, and $\textrm{Conisistent}_i = \textrm{Consistent}_{i-1}$. This completes the argument -- by induction, $M$ is enough to reconstruct the entire query/answer sequence.
\end{proof}

Finally, we can complete the proof. By Hoeffding's concentration inequality and the union bound we know that for any every sequence of queries $\phi_1,\ldots,\phi_m$ and a dataset $\bS$ of size $n$ drawn from the distribution $\calP^n$:
$$\pr\lb\exists i,\ \left|\cE_{\bS}[\phi_i]- \calP[\phi_i]\right|  \geq  \alpha \rb \leq 2m \cdot \exp\lp-2n\alpha^2\rp.$$
Applying Theorem \ref{thm:descr-basic-random} to the set $R(\phi_1,a_1,\ldots,\phi_m,a_m) = \{S \cond \exists i,\ \left|\cE_{S}[\phi_i]- \calP[\phi_i]\right|  \geq  \alpha \}$ we obtain that for the queries $\bphi_1,\ldots,\bphi_m$ generated on the dataset $\bS$ and corresponding answers of the Median Mechanism $\ba_1,\ldots,\ba_m$ we have 
\begin{eqnarray*}
\pr\lb \exists i,\ \left|\cE_{\bS}[\bphi_i]- \calP[\bphi_i]\right|  \geq  \alpha \rb &\leq& 2^b \cdot 2m\cdot  \exp\lp-2n\alpha^2\rp. \\
&\leq& 2^{\frac{\log|\X|\log m}{\alpha^2}\cdot \log (m/\alpha)}\cdot 2m\cdot \exp\lp-2n\alpha^2\rp.
\end{eqnarray*}

Solving, we have that whenever:
$$n \geq \frac{\log|\X| \cdot \log m}{2\alpha^4}\cdot \ln (m/\alpha) + \frac{\ln (2m/\beta)}{2\alpha^2} ,$$
we have: $\pr\lb \exists i,\  \left|\cE_{\bS}[\bphi_i]- \calP[\bphi_i]\right|  \geq  \alpha \rb \leq \beta.$
Combining this with Lemma \ref{lem:empirical} we have:
$$\pr\lb \exists i\in[m],\  \left|\ba_i- \calP[\bphi_i]\right|  \geq  3\alpha \rb \leq \beta.$$
Plugging in $\tau = 3\alpha$ gives the theorem.
\end{proof}
